\documentclass{article}
\usepackage[final]{nips_2018}
\usepackage[utf8]{inputenc} %
\usepackage[T1]{fontenc}    %
\usepackage{hyperref}       %
\usepackage{url}            %
\usepackage{amsfonts}       
\usepackage{nicefrac}       
\usepackage{microtype}      
\usepackage{graphicx}
\usepackage{subfigure}
\usepackage{booktabs} 
\usepackage[numbers]{natbib}
\usepackage[noend]{algorithmic}
\usepackage{algorithm}

\title{An Investigation into Neural Net Optimization via Hessian Eigenvalue Density}

\author{
  Behrooz Ghorbani\thanks{Work was done while author was an intern at Google.} \\
  Department of Electrical Engineering\\
  Stanford University\\
  \texttt{ghorbani@stanford.edu} \\
  \And
  Shankar Krishnan \\
  Machine Perception, Google Inc. \\
  \texttt{skrishnan@google.com} \\
  \And
  Ying Xiao \\
  Machine Perception, Google Inc.\\
  \texttt{yingxiao@google.com} \\
}

\usepackage{amssymb}
\usepackage{amsmath}
\usepackage{amsthm}
\usepackage{amsfonts}
\usepackage[T1]{fontenc}

\newtheorem{theorem}{Theorem}[section]

\newtheorem{lemma}[theorem]{Lemma}
\newtheorem{claim}[theorem]{Claim}

\newcommand{\tr}[1]{\mathrm{tr} \left(#1\right)}
\newcommand{\expn}[1]{\mathrm{exp}\left(#1\right)}
\newcommand{\R}{\mathbb{R}}
\newcommand{\phis}{\phi_{\sigma}}
\newcommand{\phisv}{\phis^{(v)}}
\newcommand{\loss}{\mathcal{L}}
\newcommand{\phip}{\widehat{\phi}_{poly}}

\newcommand{\Noise}{\mathcal{S}}
\newcommand{\E}[1]{\mathbb{E}\left[#1\right]}
\newcommand{\hess}{\nabla^2 \loss}

\begin{document}

\maketitle

\begin{abstract}
To understand the dynamics of optimization in deep neural networks, we develop a tool to study the evolution of the entire Hessian spectrum throughout the optimization process. Using this, we study a number of hypotheses concerning smoothness, curvature, and sharpness in the deep learning literature. We then thoroughly analyze a crucial structural feature of the spectra: in non-batch normalized networks, we observe the rapid appearance of large isolated eigenvalues in the spectrum, along with a surprising concentration of the gradient in the corresponding eigenspaces. In batch normalized networks, these two effects are almost absent. We characterize these effects, and explain how they affect optimization speed through both theory and experiments. As part of this work, we adapt advanced tools from numerical linear algebra that allow scalable and accurate estimation of the entire Hessian spectrum of ImageNet-scale neural networks; this technique may be of independent interest in other applications.
\end{abstract}

\section{Introduction}
The Hessian of the training loss (with respect to the parameters) is crucial in determining many behaviors of neural networks. The eigenvalues of the Hessian characterize the local curvature of the loss which, for example, determine how fast models can be optimized via first-order methods (at least for convex problems), and is also conjectured to influence the generalization properties. Unfortunately, even for moderate sized models, exact computation of the Hessian eigenvalues is computationally impossible. Previous studies on the Hessian have focused on small models, or are limited to computing only a few eigenvalues \cite{sagun2016eigenvalues,sagun2017empirical,yao2018hessian}. In the absence of such concrete information about the eigenvalue spectrum, many researchers have developed clever \emph{ad hoc} methods to understand notions of smoothness, curvature, sharpness, and poor conditioning in the landscape of the loss surface. Examples of such work, where some surrogate is defined for the curvature, include the debate on flat vs sharp minima \cite{keskar2016large,dinh2017sharp,wu2017towards,jastrzkebski2017three}, explanations of the efficacy of residual connections \cite{li2018visualizing} and batch normalization \cite{santurkar2018does}, the construction of low-energy paths between  different local minima \cite{draxler2018essentially}, qualitative studies and visualizations of the loss surface \cite{goodfellow2014qualitatively}, and characterization of the intrinsic dimensionality of the loss \cite{li2018measuring}. In each of these cases, detailed knowledge of the entire Hessian spectrum would surely be informative, if not decisive, in explaining the phenomena at hand.

In this paper, we develop a tool that allows us access to the entire spectrum of a deep neural network. The tool is both highly accurate (we validate it to a double-precision accuracy of $10^{-14}$ for a 15000 parameter model), and highly scalable (we are able to generate the spectra of Resnets \cite{he2016deep} and Inception V3 \cite{szegedy2016rethinking} on ImageNet in a small multiple of the time it takes to train the model). The underlying algorithm is extremely elegant, and has been known in the numerical analysis literature for decades \cite{golub1969calculation}; here we introduce it to the machine learning community, and build (and release) a system to run it at modern deep learning scale.

This algorithm allows us to peer into the optimization process with  unprecedented clarity. By generating Hessian spectra with fine time resolution, we are able to study all phases of training, and are able to comment fruitfully on a number of hypotheses in the literature about the geometry of the loss surface. Our main experimental result focuses on the role of outlier eigenvalues, we analyze how the outlier eigenvalues affect the speed of optimization; this in turn provides significant insight into how batch normalization \cite{ioffe2015batch}, one of the most popular innovations in training deep neural nets, speeds up optimization.

We believe our tool and style of analysis will open up new avenues of research in optimization, generalization, architecture design etc. So we release our code to the community to accelerate a Hessian based analysis of deep learning.

\subsection{Contributions}
In this paper, we empirically study the full Hessian spectrum of the loss function of deep neural networks. Our contributions are as follows:

In Section \ref{sec:tools}, we introduce a tool and a system, for estimating the full Hessian spectrum, capable of tackling models with tens of millions of parameters, and millions of data points. We both theoretically prove convergence properties of the underlying algorithm, and validate the system to double precision accuracy $10^{-14}$ on a toy model.

In Section \ref{sec:experiments}, we use our tool to generate Hessian spectra along the optimization trajectory of a variety of deep learning models. In doing so, we revisit a number of hypotheses in the machine learning literature surrounding curvature and optimization. With access to the entire Hessian spectrum, we are able to provide new perspectives on a variety of interesting problems: we concur with many of the coarse descriptions of the loss surface, but disagree with a number of hypotheses about how learning rate and residual connections interact with the loss surface. Our goal is not necessarily to provide proofs or refutation -- at the very least, that would require the study of a more diverse set of models -- but to provide strong evidence for/against certain interesting ideas, and simultaneously to highlight some applications of our tool.

In Section \ref{sec:batch_norm}, we observe that models with significant outlier Hessian eigenvalues exhibit slow training behavior. We provide a theoretical justification for this in Section \ref{sec:outliers} -- we argue that a non-trivial fraction of energy of the Hessian is distributed across the bulk in tiny eigenvalues, and that a coupling between the stochastic gradients and the outlier eigenvalues prevents progress in those directions. We then show that batch normalization pushes these outliers back into the bulk, and are able to isolate this effect by ablating the batch normalization operation. In Section \ref{sec:testing_hypothesis}, we confirm the predictions of our hypothesis by studying a careful intervention to batch normalization that causes the resurgence of outlier eigenvalues, and dramatic slowdowns in optimization.

\subsection{Related Work}
\label{sec:related_work}

Empirical analysis of the Hessian has been of significance interest in the deep learning community. Due to computational costs of computing the exact eigenvalues ($O(n^3)$ for an explicit $n \times n$ matrix), most of the papers in this line of research either focus on smaller models or on low-dimensional projections of the loss surface. Sagun et al.\ \cite{sagun2016eigenvalues, sagun2017empirical} study the spectrum of the Hessian for small two-layer feed-forward networks. They show that the spectrum is divided into two parts: (1) a bulk concentrated near zero which includes almost all of the eigenvalues and (2) roughly ``number of classes - 1'' outlier eigenvalues emerging from the bulk. We extend this analysis in two ways. First, we calculate the Hessian for models with $>10^7$ parameters on datasets with $>10^6$ examples -- we find that many, but not all of the above observations hold at this scale, and refine some of their observations. Secondly, we leverage the scalability of our algorithm to compute and track the Hessian spectrum throughout the optimization (as opposed to only at the end). Observing this evolution allows us to study how individual architecture choices affect optimization. There is an extensive literature regarding estimating the eigenvalues distribution of large matrices (for a small survey, see \cite{lin2016approximating}). The algorithm we use is due to Golub and Welsch \cite{golub1969calculation}. While many of these algorithms have theoretical guarantees, their empirical success is highly dependent on the problem structure. We perform a thorough comparison of our work to the recent proposal of \cite{adams2018estimating} in Appendix \ref{app:comparison}.

Batch Normalization (BN) \cite{ioffe2015batch} is one of the most influential innovations in optimizing deep neural networks as it substantially reduces the training time and the dependence of the training on initialization. There has been much interest in determining the underlying reasons for this effect. The original BN paper suggests that as the model trains, the distribution of inputs to each layer changes drastically, a phenomenon called internal covariance shift (ICS). They suggest that BN improves training by reducing ICS. There has been a series of exciting new works exploring the effects of BN on the loss surface. Santurkar et al. \ \cite{santurkar2018does} empirically show that ICS is not necessarily related to the success of the optimization. They instead prove that under certain conditions, the Lipschitz constant of the loss and $\beta$-smoothness of the loss with respect to the activations and weights of a linear layer are improved when BN is present. Unfortunately, these bounds are on a per-layer basis; this yields bounds on the diagonal blocks of the overall Hessian, but does not directly imply anything about the overall $\beta$-smoothness of the entire Hessian.
In fact even exact knowledge of $\beta$ for the entire Hessian and parameter norms (to control the distance from the optimum) is insufficient to determine the speed of optimization: in Section \ref{sec:testing_hypothesis}, we exhibit two almost identical networks that differ only in the way batch norm statistics are calculated; they have almost exactly the same largest eigenvalue and the parameters have the same scale, yet the optimization speeds are vastly different. 

During the preparation of this paper, \cite{papyan2018full} appeared on Arxiv which briefly introduces the same spectrum estimation methodology and studies the Hessian on small subsamples of MNIST and CIFAR-10 at the end of the training. In comparison, we provide a detailed exposition, error analysis and validation of the estimator in Section \ref{sec:tools}, and present optimization results on full datasets, up to and including ImageNet.

\subsection{Notation}
Neural networks are trained iteratively. We call the estimated weights at optimization iteration $t$, $\hat{\theta}_t$, $0 \leq t \leq T$. We define the loss associated with batch $i$ be $\loss_i(\theta)$. The full-batch loss is defined as $\loss(\theta) \equiv \frac{1}{N}\sum_{i=1}^N \loss_i(\theta)$ where $N$ is the number of batches.\footnote{We define the loss in terms of per-batch loss (as opposed to the per sample loss) in order to accommodate networks that use batch normalization.} The Hessian, $\nabla^2 \loss(\theta) \in \R^{n \times n}$ is a symmetric matrix such that 
$\hess(\theta)_{i, j} = \frac{\partial^2}{\partial_{\theta_i}\partial_{\theta_j}} \loss(\theta)$.
Note that our Hessians are all ``full-batch'' Hessians (i.e., they are computed using the entire dataset). When there is no confusion, we represent $\nabla^2 \loss(\hat{\theta}_t)$ with $H\in \R^{n\times n}$. Throughout the paper, $H$ has the spectral decomposition $Q \Lambda Q^T$ where $\Lambda = diag(\lambda_1,\dots, \lambda_n)$, $Q=[q_1, \dots, q_n]$ and $\lambda_1 \geq \lambda_2 \cdots \geq \lambda_n$. 

\section{Accurate and Scalable Estimation of Hessian Eigenvalue Densities for $n > 10^7$}
\label{sec:tools}

To understand the Hessian, we would like to compute the eigenvalue (or spectral) density, defined as $\phi(t) = \frac{1}{n} \sum_{i=1}^n \delta(t - \lambda_i)$ where $\delta$ is the Dirac delta operator. The naive approach requires calculating $\lambda_i$; however, when the number of parameters, $n$, is large this is not tractable. We relax the problem by convolving with a Gaussian density of variance $\sigma^2$ to obtain:
\begin{align}\label{eqn:smoothed_density}
  \phis(t) = \frac{1}{n} \sum_{i=1}^n f(\lambda_i; t, \sigma^2)
\end{align}
where $f(\lambda; t, \sigma^2) =\frac{1}{\sigma \sqrt{2 \pi}} \expn{- \frac{(t - \lambda)^2}{2 \sigma^2}}.$ 
For small enough $\sigma^2$, $\phis(t)$ provides all practically relevant information regarding the eigenvalues of $H$. Explicit representation of the Hessian matrix is infeasible when $n$ is large, but using Pearlmutter's trick \cite{pearlmutter1994fast} we are able to compute Hessian-vector products for any chosen  vector. 

\subsection{Stochastic Lanczos Quadrature}
It has long been known in the numerical analysis literature that accurate stochastic approximations to the eigenvalue density can be achieved with much less computation than a full eigenvalue decomposition. In this section, we describe the \emph{stochastic Lanczos quadrature} algorithm \cite{golub1969calculation}. Although the algorithm is already known, its mathematical complexity and potential as a research tool warrant a clear exposition for a machine learning audience.  We give the pseudo-code in Algorithm \ref{meta_alg}, and describe the individual steps below, deferring a discussion of the various approximations to Section \ref{sec:high_accuracy}.

Since $H$ is diagonalizable and $f$ is analytic, we can define $f(H) = Q f(\Lambda) Q^T$ where $f(\cdot)$ acts point-wise on the diagonal of $\Lambda$. Now observe that if $v \sim N(0, \frac{1}{n} I_{n \times n})$, we have
\begin{align}
\phis(t) = \frac{1}{n} \tr{f(H, t, \sigma^2)} = \E{v^T f(H, t, \sigma^2) v} \label{eqn:smoothed}
\end{align}
Thus, as long as $\phisv(t) \equiv v^T f(H, t, \sigma^2) v$ concentrates fast enough, to estimate $\phis(t)$, it suffices to sample a small number of random $v$'s and average $\phisv(t)$.

\begin{algorithm}[H] 
Draw $k$ i.i.d realizations of $v$, $\{v_1 , \dots, v_k\}$.\;
\begin{itemize}
\item[I.]  Estimate $\phis^{(v_i)}(t)$ by a quantity $\widehat{\phi}^{(v_i)}(t)$:
\begin{itemize}
    \item Run the Lanczos algorithm for $m$ steps on matrix $H$ starting from $v_i$ to obtain tridiagonal matrix $T$.
    \item Compute eigenvalue decomposition $T = ULU^T$.
    \item Set the nodes $\ell_i = (L_{ii})_{i=1}^m$ and weights $\omega_i = (U^2_{1,i})_{i=1}^m$.
    \item Output $\widehat{\phi}^{(v_i)}(t) = \sum_{i=1}^m \omega_i f(\ell_i; t, \sigma^2)$.
\end{itemize}
\item[II.] Set $\widehat{\phi}_{\sigma}(t) = \frac{1}{k}\sum_{i=1}^k \widehat{\phi}^{(v_i)}(t)$.
\end{itemize}
\caption{Two Stage Estimation of $\phis(t)$} \label{meta_alg}
\end{algorithm}

By definition, we can write
\begin{align} \label{eqn:beta_sum}
\phisv(t) = v^T Q f(\Lambda; t, \sigma^2) Q^T v &= \sum_{i=1}^n (v^T q_i)^2 f(\lambda_i; t, \sigma^2) \nonumber \\
&= \sum_{i=1}^n \beta_i^2 f(\lambda_i; t, \sigma^2)
\end{align}
where $\beta_i \equiv (v^T q_i)$. Instead of summing over the discrete index variable $i$, we can rewrite this as a Riemann-Stieltjes integral over a continuous variable $\lambda$ weighted by $\mu$:
\begin{align} \label{eqn:integral}
\phisv(t) = \int_{\lambda_n}^{\lambda_1} f(\lambda; t, \sigma^2) d\mu(\lambda)
\end{align}
where $\mu$ is a CDF (note that the probability density $d \mu$ is a sum of delta functions that directly recovers Equation \ref{eqn:beta_sum})\footnote{Technically $\mu$ is a positive measure, not a probability distribution, because $||v||^2$ only concentrates on 1. This wrinkle is irrelevant.}.
\begin{align*}
  \mu(\lambda) = \begin{cases} 0 & \lambda < \lambda_n \\
  \sum_{i=1}^k \beta_i^2 & \lambda_k \le \lambda < \lambda_{k+1} \\
   \sum_{i=1}^n \beta_i^2 & \lambda \ge \lambda_1\end{cases}.
\end{align*}

To evaluate this integral, we apply a quadrature rule (a quadrature rule approximates an integral as a weighted sum -- the well-known high-school trapezoid rule is a simple example). In particular, we want to pick a set of weights $\omega_i$ and a set of nodes $l_i$ so that
\begin{align}\label{eqn:approx_form}
  \phisv(t) \approx \sum_{i=1}^m \omega_i f(\ell_i; t, \sigma^2) \equiv \widehat{\phi}^{(v)}(t)
\end{align}
The hope is that there exists a good choice of $(\omega_i, \ell_i)_{i=1}^m$ where $m \ll n$ such that $\phisv(t)$ and $\widehat{\phi}^{(v)}(t)$ are close for all $t$, and that we can find the nodes and weights efficiently for our particular integrand $f$ and the CDF $\mu$. The construction of a set of suitable nodes and weights is a somewhat complicated affair. It turns out that if the integrand were a polynomial $g$ of degree $d$, with $d$ small enough compared to $m$, it is possible to compute the integral exactly,
\begin{align}\label{eqn:exact}
    \int g d\mu = \sum_{i=1}^m w_i g(l_i).
\end{align}
\begin{theorem}[\cite{golub2009matrices} Chapter 6]\label{thm:gauss_quad_opt}
Fix $m$. For all $(\beta_i, \lambda_i)_{i=1}^n$, there exists an approximation rule generating node-weight pairs $(\omega_i, \ell_i)_{i=1}^m$ such that for any polynomial, $g$ with $deg(g) \leq 2m-1$, \eqref{eqn:exact} is true. This approximation rule is called the \textbf{Gaussian quadrature}. The degree $2m - 1$ achieved is maximal: for a general $(\beta_i, \lambda_i)_{i=1}^n$, no other approximation rule can guarantee exactness of Equation \eqref{eqn:exact} for higher degree polynomials. 
\end{theorem}

The Gaussian quadrature rule always generates non-negative weights. Therefore, as $f(\cdot; t, \sigma) \geq 0$, it is guaranteed that $\widehat{\phi} \geq 0$ which is a desirable property for a density estimate. For these reasons, despite the fact that our integrand $f$ is not a polynomial, we use the Gaussian quadrature rule. For the construction of the Gaussian quadrature nodes and weights, we rely on a deep connection between Gaussian quadrature and Krylov subspaces via orthogonal polynomials. We refer the interested reader to the excellent \cite{golub2009matrices} for this connection.
\begin{theorem}[\cite{golub1969calculation}] \label{thm:naive_computation}
Let $V = [v, Hv, \cdots, H^{m-1}v] \in \R^{n \times m}$ and $\tilde{V}$ be the incomplete basis resulting from applying QR factorization on $V$. Let $T \equiv \tilde{V}^TH\tilde{V} \in \R^{m \times m}$ and $ULU^T$ be the spectral decomposition of $T$. Then the Gaussian quadrature nodes $\ell_i$ are given by $(L_{i, i})_{i=1}^m$, and the Gaussian quadrature weights $\omega_i$ are given by $(U_{1, i}^2)_{i=1}^m$.
\end{theorem}

Theorem \ref{thm:naive_computation} presents a theoretical way to compute the Gaussian quadrature rule (i.e., apply the $H$ matrix repeatedly and orthogonalize the resulting vectors). There are well-known algorithms that circumvent calculating the  numerically unstable $V$, and compute $T$ and $\tilde{V}$ directly. We use Lanczos algorithm \cite{lanczos1950iteration} (with full re-orthogonalization) to perform this computation in a numerically stable manner.

\subsection{Accuracy of Gaussian Quadrature Approximation}
\label{sec:high_accuracy}
 Intuition suggests that as long as $f(\cdot; t, \sigma^2)$ is close to some polynomial of degree at most $2m-1$, our approximation must be accurate (i.e., Theorem \ref{thm:gauss_quad_opt}). Crucially, it is not necessary to know the exact approximating polynomial, its mere existence  is sufficient for an accurate estimate. There exists an extensive literature on bounding this error; \cite{ubaru2017fast} prove that under suitable conditions that
\begin{align} \label{eqn:theoretical_bound}
\vert \widehat{\phi}^{(v)}(t) - \phisv (t)\vert \leq c \frac{1}{(\rho^2 - 1) \rho^{2m}}
\end{align}
where $\rho > 1$. The constant $\rho$ is closely tied to how well $f(\cdot;t, \sigma^2)$ can be approximated by Chebyshev polynomials. \footnote{We refer the interested reader to \cite{ubaru2017fast, demanet2010chebyshev} for more details} In our setting, as $\sigma^2$ decreases, higher-order polynomials become necessary to approximate $f$ well. Therefore, as $\sigma^2$ decreases, $\rho$ decreases and more Lanczos iterations become necessary to approximate the integral well. 

To establish a suitable value of $m$, we perform an empirical analysis of the error decay when $H$ corresponds to a neural network loss Hessian. In Appendix \ref{app:verify}, we study this error on a 15910 parameter feed-forward MNIST network, where the model is small enough that we can compute $\phisv(t)$ exactly. For $\sigma^2 = 10^{-5}$, a quadrature approximation of order $80$ achieves maximum double-precision accuracy of $10^{-14}$. Following these results, we use $\sigma^2 = 10^{-5}, m = 90$ for our experiments. Equation \ref{eqn:theoretical_bound} implies that the error decreases exponentially in $m$, and since GPUs are typically run in single precision, our $m$ is an extremely conservative choice. 

\subsection{Concentration of the Quadratic Forms} 
\label{sec:concentration}

Although $\phisv(\cdot)$ is an unbiased estimator for $\phis(\cdot)$, we must still study its concentration towards its mean. We prove:

\begin{claim}\label{cor:phi_concentrate}
Let $t$ be a fixed evaluation point and $k$ be the number of realizations of $v$ in step II. Let $a = \Vert f(H; t, \sigma^2) \Vert_F$ and $b = \Vert f(H; t, \sigma^2) \Vert_2$. Then for any $x>0$, 
\begin{align*} 
P\bigg(|\phis(t) - \widehat{\phi}_{\sigma}(t)| > \frac{2a}{n\sqrt{k}} \sqrt{x} +  \frac{2b}{k n} x \bigg) \leq 2\exp(-x).
\end{align*}

Alternatively, since $f(\cdot)$ is a Gaussian density, we can give norm independent bounds:  $\forall x > 0$, 
\begin{align} \label{eq:worst_case}
P\bigg(|\phis(t) - \widehat{\phi}_{\sigma}
(t)| > \epsilon(x) \bigg) \leq 2\exp(-x).
\end{align}
where $\epsilon(x) \equiv \sqrt{\frac{2}{\pi \sigma^2}} (\sqrt{\frac{x}{nk}} + \frac{x}{nk})$.
\end{claim}

Claim \eqref{cor:phi_concentrate} shows that $\widehat{\phi}_{\sigma}(t)$ concentrates exponentially fast around its expectation. Note in particular the $\sqrt{n}$ and higher powers in the denominator -- since the number of parameters $n > 10^6$ for cases of interest, we expect the deviations to be negligible. We plot these error bounds and prove Claim \ref{cor:phi_concentrate} in Appendix \ref{app:more_concentration}.

\subsection{Implementation, Validation and Runtime}
We implemented a large scale version of Algorithm \ref{meta_alg} in TensorFlow \cite{abadi2016tensorflow}; the main component is a distributed Lanczos Algorithm. We describe the implementation and performance in Appendix \ref{app:implementation}. To validate our system, we computed the exact eigenvalue distribution on the 15910 parameter MNIST model. Our proposed framework achieves $L_1(\phis, \widehat{\phi}_{\sigma}) \equiv \int_{-\infty}^{\infty} |\phis(t) -  \widehat{\phi}(t)|dt \approx 0.0012$ which corresponds to an extremely accurate solution. The largest model we've run our algorithm on is Inception V3 on ImageNet. The runtime is dominated by the application of the Hessian-vector products within the Lanczos algorithm; we run $O(mk)$ full-batch Hessian vector products. The remaining cost of the Lanczos algorithm is negligible at  $O(km^2n)$ floating point operations. For a Resnet-18 on ImageNet, running a single draw takes about half the time of training the model.

\begin{figure}[h]
\includegraphics[width=\textwidth]{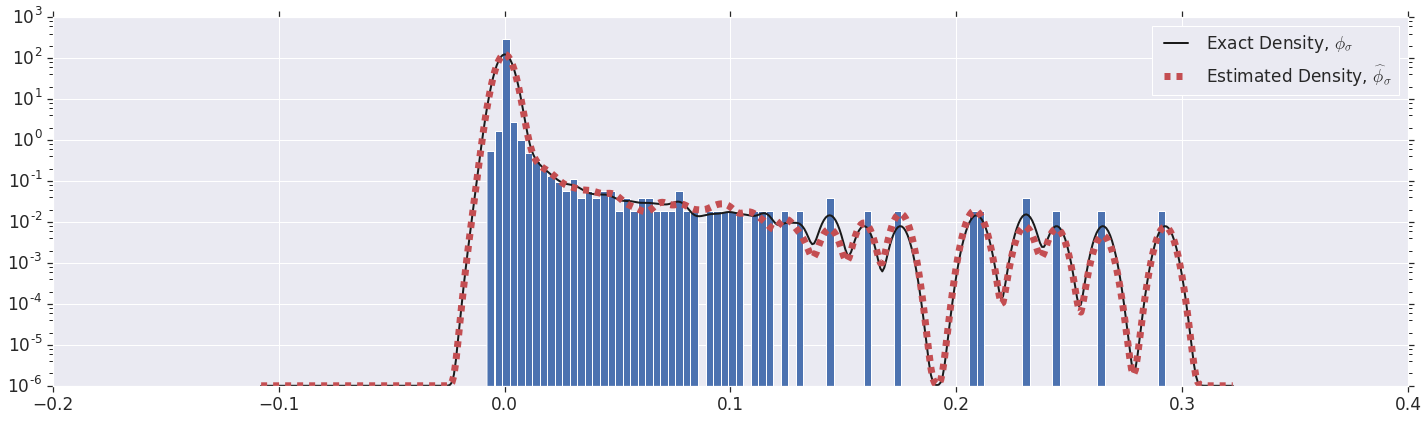}
\vspace{-0.5cm}
\caption{Comparison of the estimated smoothed density (dashed) and the exact smoothed density (solid) in the interval $[-0.2, 0.4]$. We use $\sigma^2 = 10^{-5}, k=10$ and degree $90$ quadrature. For completeness, the histogram of the exact eigenvalues is also plotted. \label{fig:verification}}
\end{figure}

In Appendix \ref{app:comparison}, we compare our approach to a recent proposal \cite{adams2018estimating} to use Chebyshev approximation for estimating the spectral density.

\section{Spectral densities throughout optimization} \label{sec:experiments}

The tool we developed in Section \ref{sec:tools} gives us an unprecedented ability to examine the loss landscape of deep neural networks. In particular, we can track the spectral density throughout the entire optimization process. Our goal in this section is to provide direct curvature evidence for (and against) a number of hypotheses about the loss surface and optimization in the literature. We certainly can not conclusively prove or refute even a single hypothesis within the space constraints, but we believe that the evidence is very strong in many of these cases.

For our analysis, we study a variety of Resnet and VGG \cite{simonyan2014very} architectures on both CIFAR-10 and ImageNet. Details are presented in Appendix \ref{app:details}. The Resnet-32 on CIFAR-10 has $4.6 \times 10^5$ parameters; all other models have at least $10^7$. For the sake of consistency, our plots in this section are of Resnet spectral densities; we have reproduced all these results on non-residual (VGG) architectures.

At initialization, we observe that large negative eigenvalues dominate the spectrum. However, as Figure \ref{fig:negative_eigenvalues} shows, in only very few steps ($<1\%$ of the total number of steps; we made no attempt to optimize this bound), these large negative eigenvalues disappear and the overall shape of the spectrum stabilizes. Sagun et al. \cite{sagun2016eigenvalues} had observed a similar disappearance of negative eigenvalues for toy feed-forward models after the training, but we are able to pinpoint this phase to the very start of optimization. This observation is readily reproducible on ImageNet.

\begin{figure}[ht]
\includegraphics[width=\textwidth]{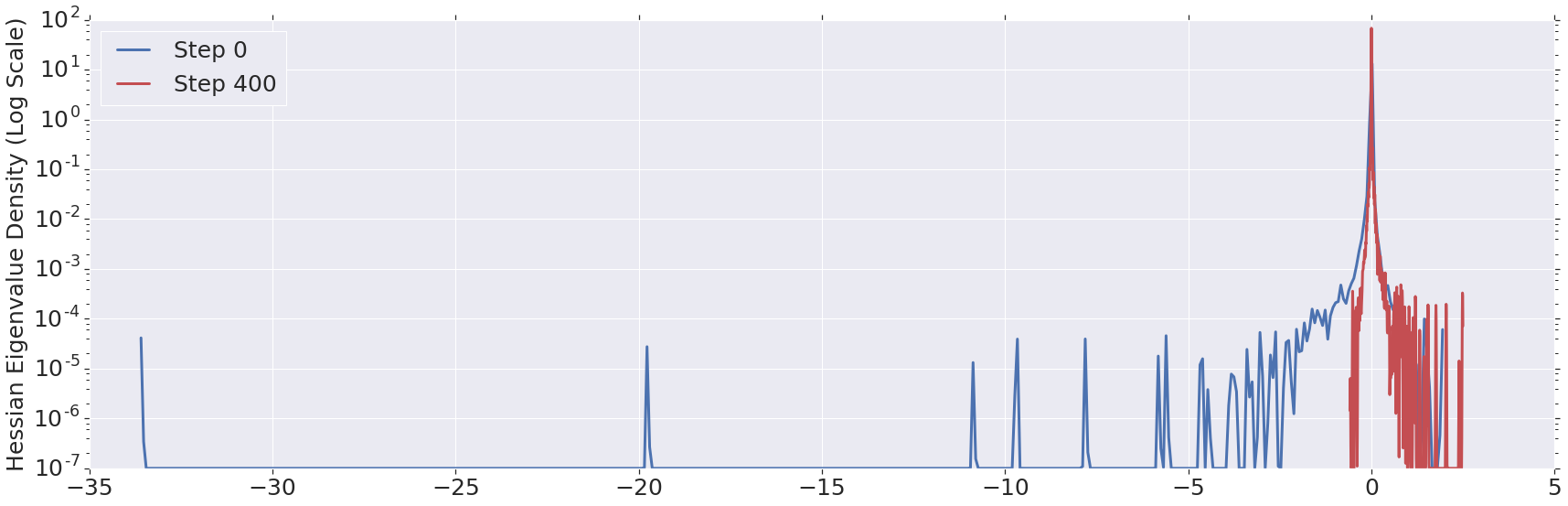}
\vspace{-0.5cm}
\caption{The evolution of the spectrum of a Resnet-32 in the beginning of training. After just $400$ momentum steps, large negative eigenvalues disappear. \label{fig:negative_eigenvalues}}
\end{figure}

Throughout the rest of the optimization, the spectrum is almost entirely flat, with the vast majority ($>99.99\%$ of eigenvalues being close to 0). This is in accordance with the ideas of Li et al. \cite{li2018measuring}, who hypothesize that the loss surface has low intrinsic dimensionality, and also with results of Sagun et al.\ on toy models. In the case of $K$-class classification with small two-layer feed-forward networks, Sagun et al.\  had observed that the Hessian spectrum contains roughly $K$ outliers which are a few orders of magnitudes larger than the rest of the eigenvalues. Contrary to this, we find that the emergence of these outliers is highly dependent on whether BN is present in the model or not. We study this behavior in depth in Section \ref{sec:batch_norm}.

Also in Sagun et al.\ is the observation that the negative eigenvalues at the end of the training are orders of magnitude smaller than the positive ones. While we are able to observe this on CIFAR-10, what
happens on ImageNet seems to be less clear (Figure \ref{fig:image_net_negatives}). We can derive a useful metric by integrating the spectral densities. At the end of optimization, the total $L_1$ energy of the negative eigenvalues is comparable to that of the positive eigenvalues (0.434 vs 0.449), and the $L_2$ energy is smaller, but still far from zero (0.025 vs 0.036). In comparison, on CIFAR-10 the $L_2$ energies are 0.025 and 0.179 in the negative and positive components respectively. We believe that the observation of Sagun et al.\ may be an artifact of the tiny datasets used -- on MNIST and CIFAR-10 one can easily attain zero classification loss (presumably a global minimum); on ImageNet, even a much larger model will fail to find a zero loss solution.
\begin{figure}[ht]
\includegraphics[width=\textwidth]{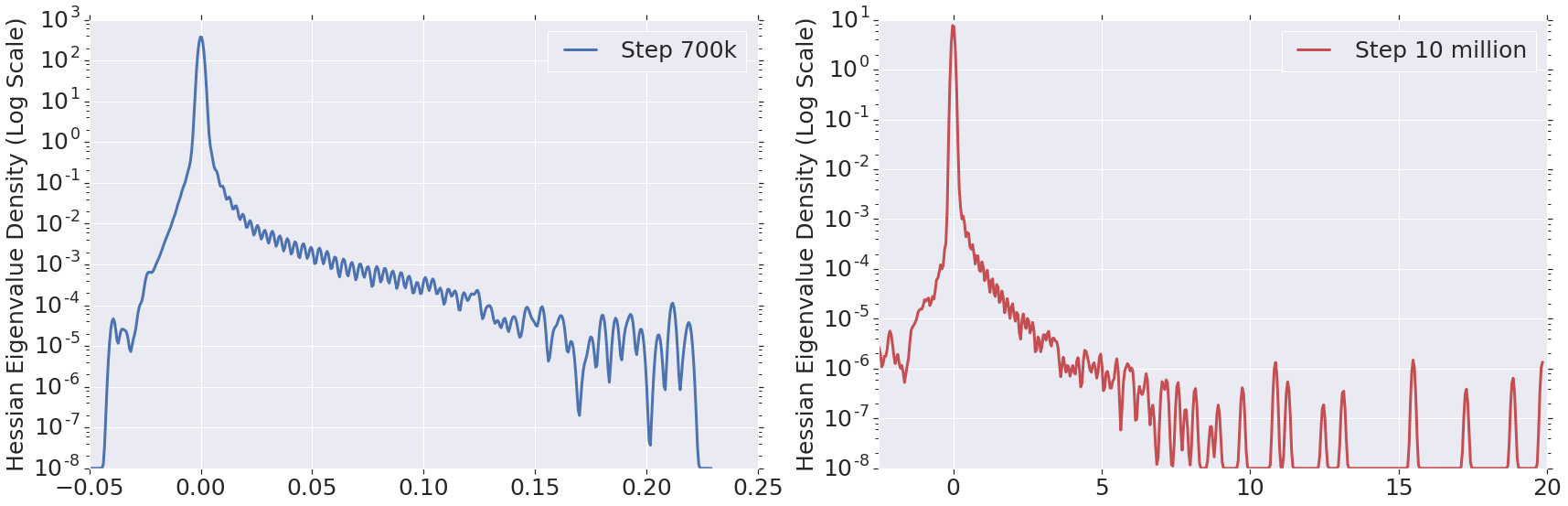}
\vspace{-0.5cm}
\caption{Spectral densities of Resnet-18 on ImageNet towards the start, and at the end of optimization. There is a notable negative density towards the end of optimization.
 \label{fig:image_net_negatives}}
\end{figure}

Jastrzkebski et al.\ \cite{jastrzkebski2017three}, building on a line of work surrounding flat and sharp minima, hypothesized that lower learning rates correspond to sharper optima. We consider this question by inspecting the spectral densities immediately preceding and following a learning rate drop. According to the hypothesis, we would then expect the spectral density to exhibit more extremal eigenvalues. In fact, we find the exact opposite to be true in Figure \ref{fig:rate_drop} -- not only do the large eigenvalues contract substantially after the learning rate drop at 40k steps, we have a lower density at all values of $\lambda$ except in a tiny ball around 0. This is an extremely surprising result, and violates the common intuition that lower learning rates allow one to slip into small, sharp crevices in the loss surface. We note that this is not a transient phenomenon -- the spectrum before and afterwards are stable over time.

\begin{figure}[ht]
\includegraphics[width=\textwidth]{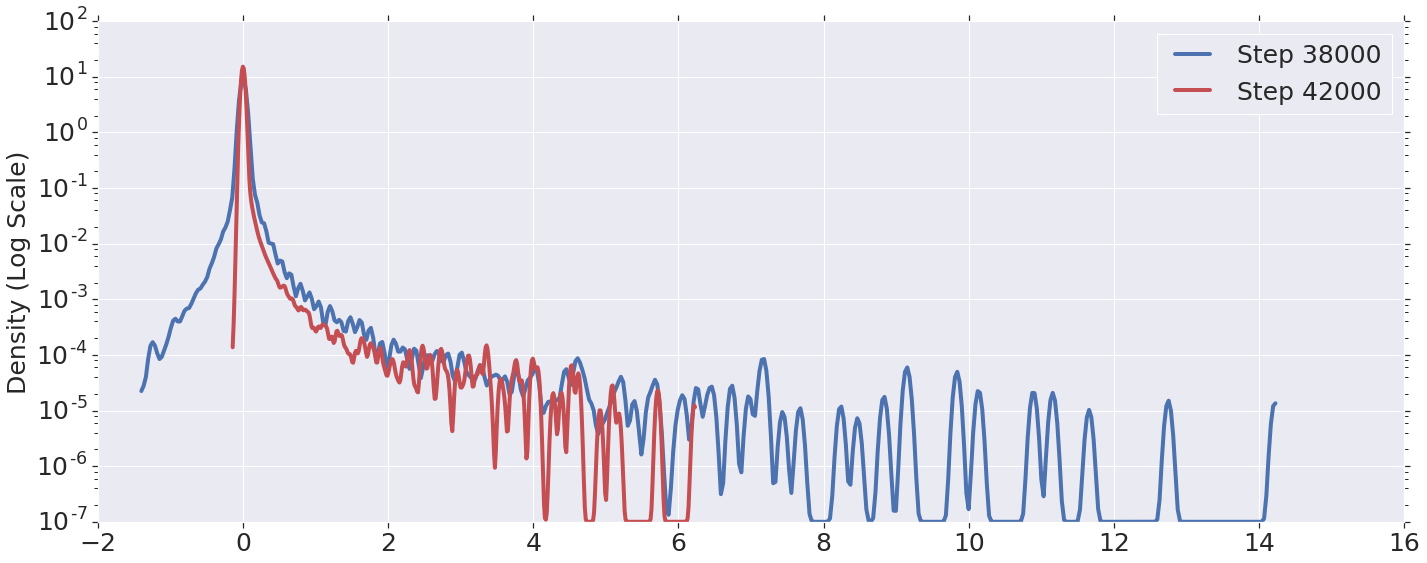}
\vspace{-0.5cm}
\caption{Spectral densities of Resnet-32 preceding and following a learning rate decrease (at step 40000). The Hessian prior to the learning rate drop appears sharper. \label{fig:rate_drop}}
\end{figure}

Finally, Li et al.\ \cite{li2018visualizing} recently hypothesized that adding residual connections significantly smooths the optimization landscape, producing a series of compelling two-dimensional visualizations. We compared a Resnet-32 with and without residual connections, and we observe in Figure \ref{fig:residual} that all eigenvalues contract substantially towards zero. This is contrary to the visualizations of Li et al.

\begin{figure}[ht]
\includegraphics[width=\textwidth]{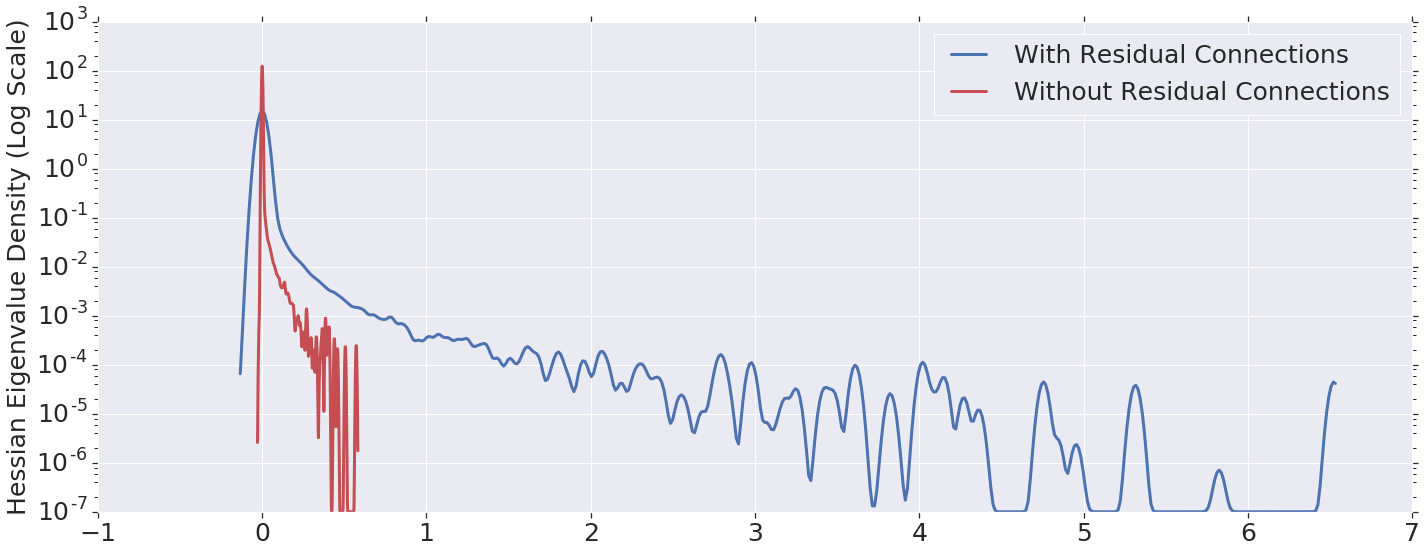}
\vspace{-0.5cm}
\caption{Spectral densities of Resnet-32 with and without residual connections (at step 40000). The Hessian without residual connections appears to be smoother. \label{fig:residual}}
\end{figure}

\section{Outlier Eigenvalues Slow Optimization; Batch Norm Suppresses Outliers} \label{sec:batch_norm}

In some of the spectral densities presented so far, perhaps the most salient feature is the presence of a small number of outlier eigenvalues that are located far from the bulk of the spectrum. We noticed that these outliers are much larger and much further from the bulk for some architectures than others (i.e., for VGG the outliers are extremely far, less so for Resnets). Suspecting that batch normalization was the crucial difference, we ran a series of ablation experiments contrasting the spectral density in the presence and absence of batch normalization (i.e., we added BN to models that did not already have it, and removed BN from models that already did). Figure \ref{fig:comparison_bn_Resnet} contrasts the the Hessian spectrum in the presence of BN vs the spectrum when BN is removed. The experiment yields the same results on VGG  on CIFAR-10 (Figure \ref{fig:comparison_bn_vgg}), and Resnet-18 on ImageNet (Figure \ref{fig:comparison_bn_Resnet_imagenet}), and at various points through training.

Our experiments reveal that, in the presence of BN, the largest eigenvalue of the Hessian, $\lambda_1(H)$ tend to not to deviate as much from the bulk. In contrast, in non-BN networks, the outliers grow much larger, and further from the bulk. To probe this behavior further we formalize the notion of an outlier with a metric:
\begin{align*}
\zeta(t) := \frac{\lambda_1(\nabla^2 \loss (\theta_t))}{\lambda_{K}(\nabla^2 \loss (\theta_t))}.
\end{align*}
This provides a scale-invariant measure of the presence of outliers in the spectrum. In particular, if $K - 1$ (as suggested by Sagun et al.\ \cite{sagun2016eigenvalues,sagun2017empirical} outliers are present in the spectrum, we expect $\zeta \gg 1$. Figure \ref{fig:outlier} plots $\zeta(t)$ throughout training. It is evident that \emph{relative} large eigenvalues appear in the spectrum. Normalization layer induces an odd dependency on parameter scale -- scaling the (batch normalized) weights leads to unchanged activations, and inversely scales the gradients. Obviously, we can not conclude that the problem is much easier! Thus, for studying the optimization performance of batch normalization, we must have at least a global scaling invariant quantity -- which $\zeta(t)$ is. In contrast, the analysis in \cite{santurkar2018does} varies wildly with scale\footnote{We have also tried normalizing individual weights matrices and filters, but this leads to blowup in some gradient components.}.

Informed by the experimental results in this section, we hypothesize a mechanistic explanation for why batch normalization speeds up optimization: it does so via suppression of outlier eigenvalues which slow down optimization.

\begin{figure}[ht]
\includegraphics[width=\textwidth]{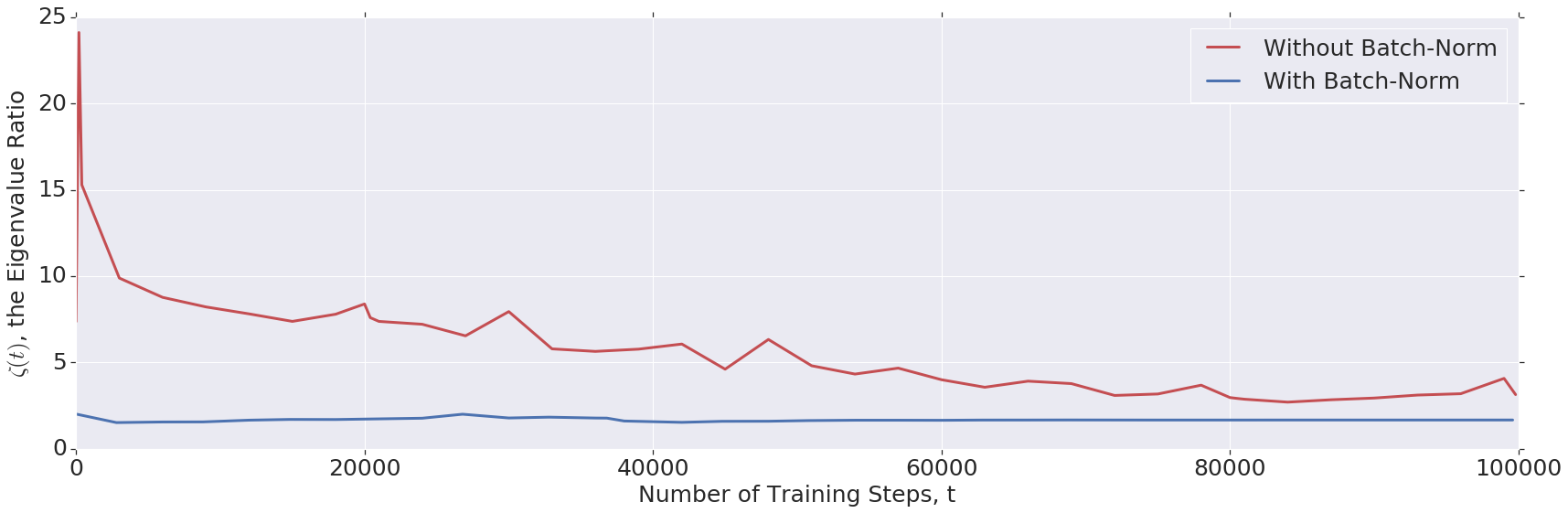}
\vspace{-0.5cm}
\caption{$\zeta(t)$ for Resnet-32 throughout training. The model without BN (red) consistently shows significantly higher eigenvalue fraction. \label{fig:outlier}}
\end{figure}

\begin{figure}[ht]
\includegraphics[width=\textwidth]{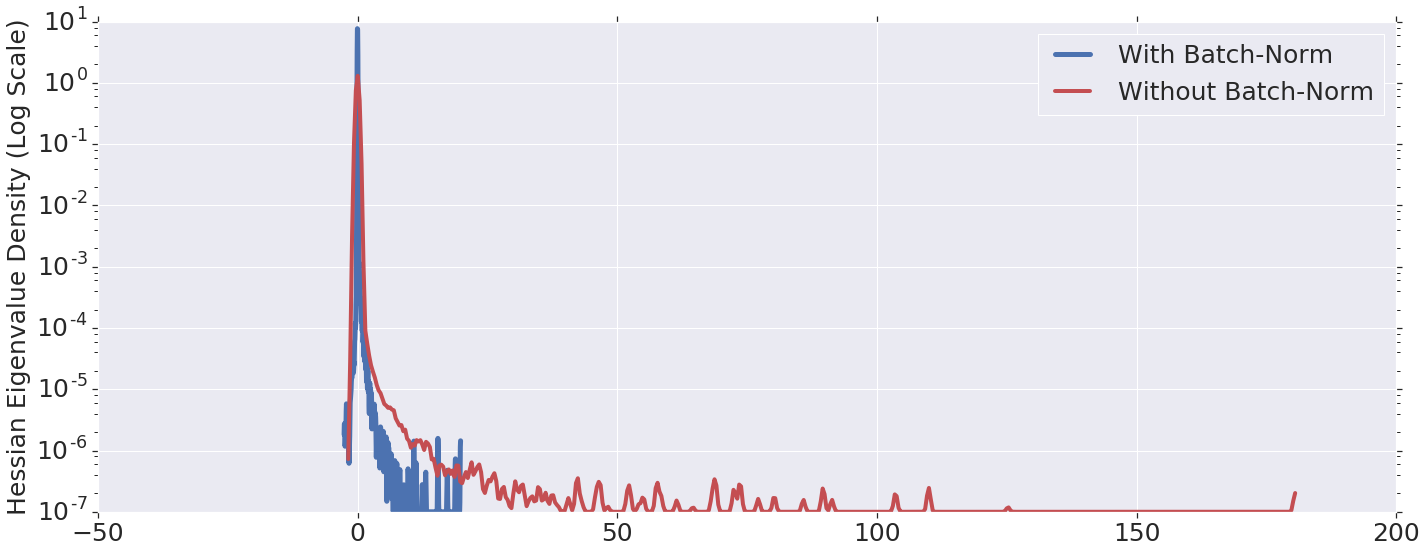}
\vspace{-0.5cm}
\caption{The eigenvalue comparison of the Hessian of Resnet-18 trained on ImageNet dataset. Model with BN is shown in blue and the model without BN in red. The Hessians are computed at the end of training. \label{fig:comparison_bn_Resnet_imagenet}}
\end{figure}

\begin{figure}[ht]
\includegraphics[width=\textwidth]{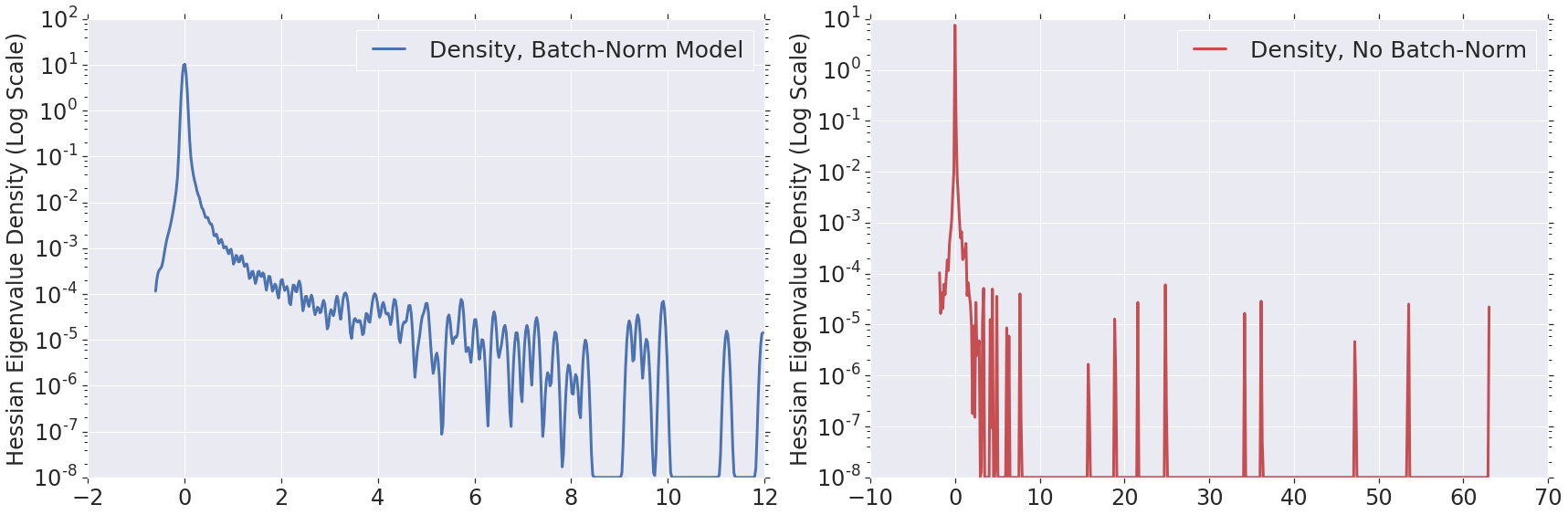}
\vspace{-0.5cm}
\caption{The eigenvalue comparison of the Hessian of the Resnet-32 model with BN (blue) and without BN (red). To allow comparison on the same plot, the densities have been normalized by their respective $10^{th}$ largest eigenvalue. The Hessians are computed after $48k$ steps of training. \label{fig:comparison_bn_Resnet}}
\end{figure}

\begin{figure}[ht]
\includegraphics[width=\textwidth]{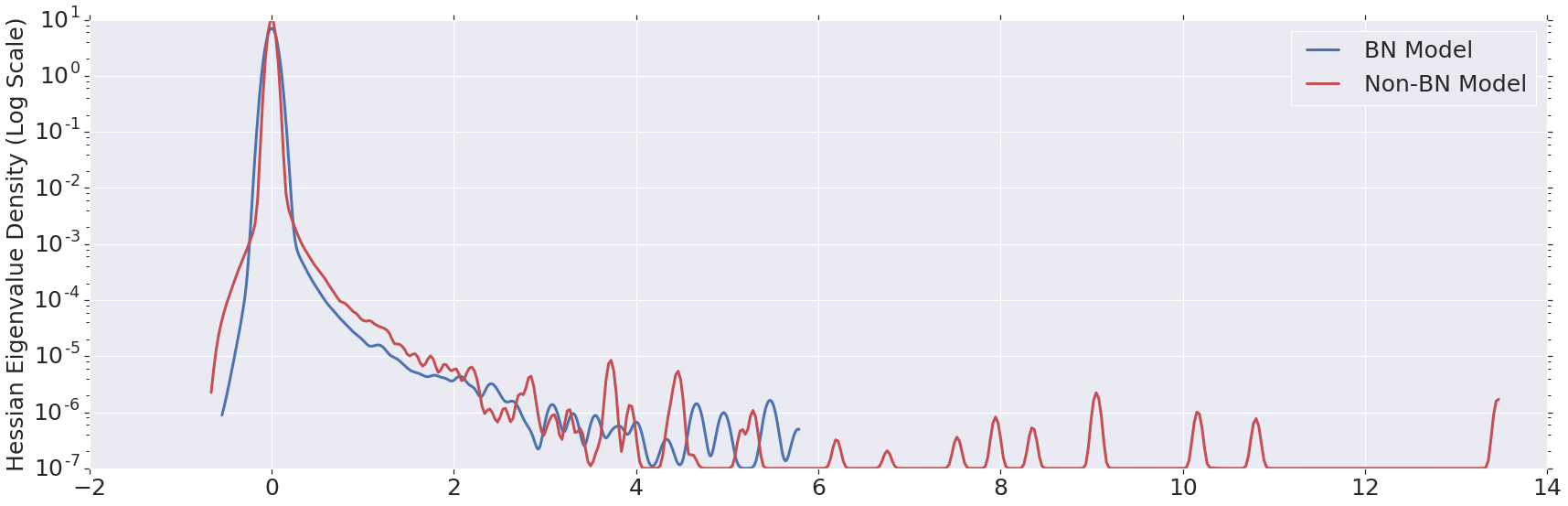}
\vspace{-0.5cm}
\caption{The eigenvalue comparison of the Hessian of the VGG network with BN (blue) and without BN (red). The Hessians are computed after $5058$ steps of training. \label{fig:comparison_bn_vgg}}
\end{figure}

\subsection{Mechanisms by which outliers slow optimization}
\label{sec:outliers}

In this section, we seek to answer the question ``\emph{Why} do outlier eigenvalues slow optimization?'' One answer to this question is obvious. Large $\lambda_1$ implies that one must use a very low learning rate; but this an incomplete explanation -- $\lambda_1$ has to be large \emph{with respect to the rest of the spectrum}. To make this explicit, consider a simple quadratic approximation to the loss around the optimum, $\theta^*$:
\begin{align}
\loss(\theta) \approx \loss(\theta^*) + \frac{1}{2} (\theta - \theta^*)^TH(\theta - \theta^*)
\end{align}
where without loss of generality, we assume $ H = diag(\lambda_1, \cdots, \lambda_n)$ with $\lambda_i >0$. We can easily show that when optimized with gradient descent with a learning rate $\eta < 2 / \lambda_i$ sufficiently small for convergence that in the eigenbasis, we have:
\begin{align}\label{eqn:quadratic_convergence}
|\hat{\theta}_t - \theta^*|_i \leq \bigg\vert 1 - \frac{2\lambda_i}{\lambda_1} \bigg\vert^t |\hat{\theta}_0 - \theta^*|_i
\end{align}
For all directions where $\lambda_i$ is small with respect to $\lambda_1$, we expect convergence to be slow. One might hope that these small $\lambda_i$ do not contribute significantly to the loss; unfortunately, when we measure this in a Resnet-32 with no batch normalization, a small ball around 0 accounts for almost 50\% of the total $L_1$ energy of the Hessian eigenvalues for a converged model (the $L_1$ reflects the loss function $\sum_i \lambda_i (\theta - \theta^{\ast})_i^2$). Thus to achieve successful optimization, we are forced to optimize these slowly converging directions\footnote{While the loss function in deep nets is not quadratic, the intuition that the result above provides is still valid in practice.}.

A second, more pernicious reason lies in the interaction between the large eigenvalues of the Hessian and the stochastic gradients. Define the covariance of the (stochastic) gradients at time $t$ to be 
\begin{align}
\Sigma(t) \equiv \frac{1}{N} \sum_{i = 1}^N \nabla \loss_i \nabla \loss_i^T.
\end{align}
The eigenvalue density of $\Sigma$ characterizes how the energy of the (mini-batch) gradients is distributed (the tools of Section \ref{sec:tools} apply just as well here). As with the Hessian, we observe that in non-BN networks the spectrum of $\Sigma$ has outlier eigenvalues (Figure \ref{fig:covariance}). Throughout the optimization, we observe that almost all of the gradient energy is concentrated in these outlier subspaces (Figure \ref{fig:gradient_figure}), reproducing an observation of Gur-Ari et al.\ \cite{gur2018gradient}\footnote{In addition, we numerically verify that the outlier subspaces of $H$ and $\Sigma$ mostly coincide: throughout the optimization, for a Resnet-32, $99\%$ of the energy of the outlier Hessian eigenvectors lie in the outlier subspace of $\Sigma(t)$.}. We observe that when BN is introduced in the model, this concentration subsides substantially.

\begin{figure}[h]
\includegraphics[width=\textwidth]{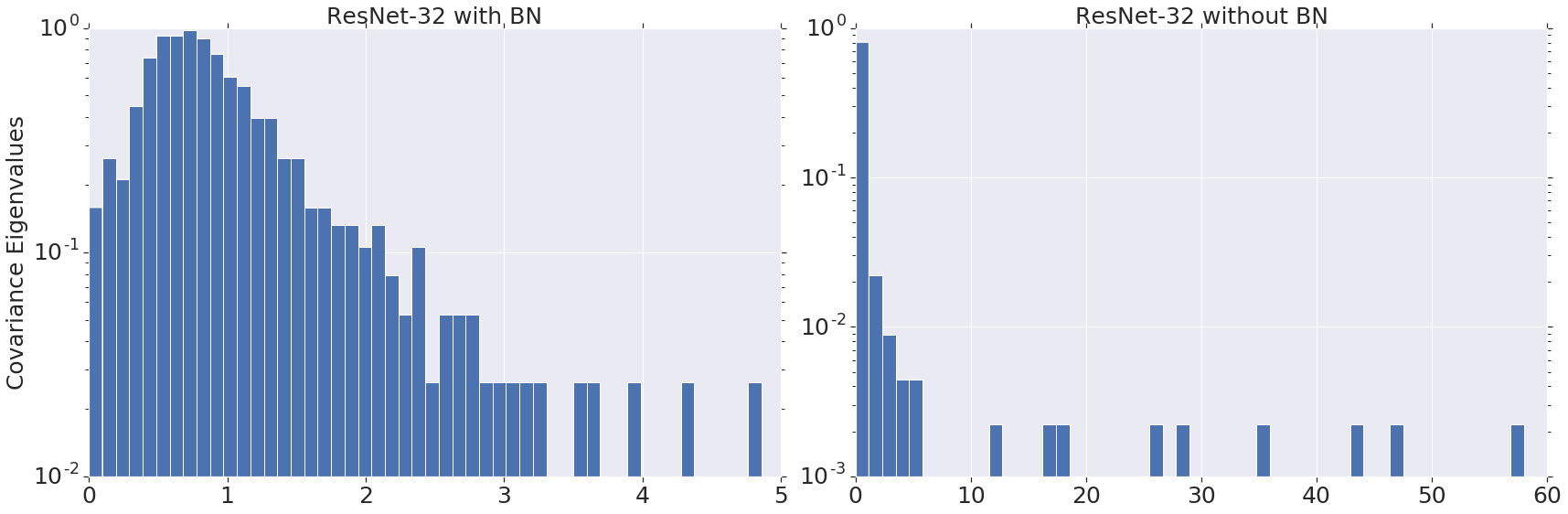}
\vspace{-0.5cm}
\caption{The histogram of the eigenvalues of $\Sigma$ for a Resnet-32 with (left) and without (right) BN after $9k$ training steps. In no BN case, almost $99\%$ of the energy is in the top few subspaces. For easier comparison, the distributions are normalized to have the same mean. \label{fig:covariance}}
\end{figure}

\begin{figure}[h]
\includegraphics[width=\textwidth]{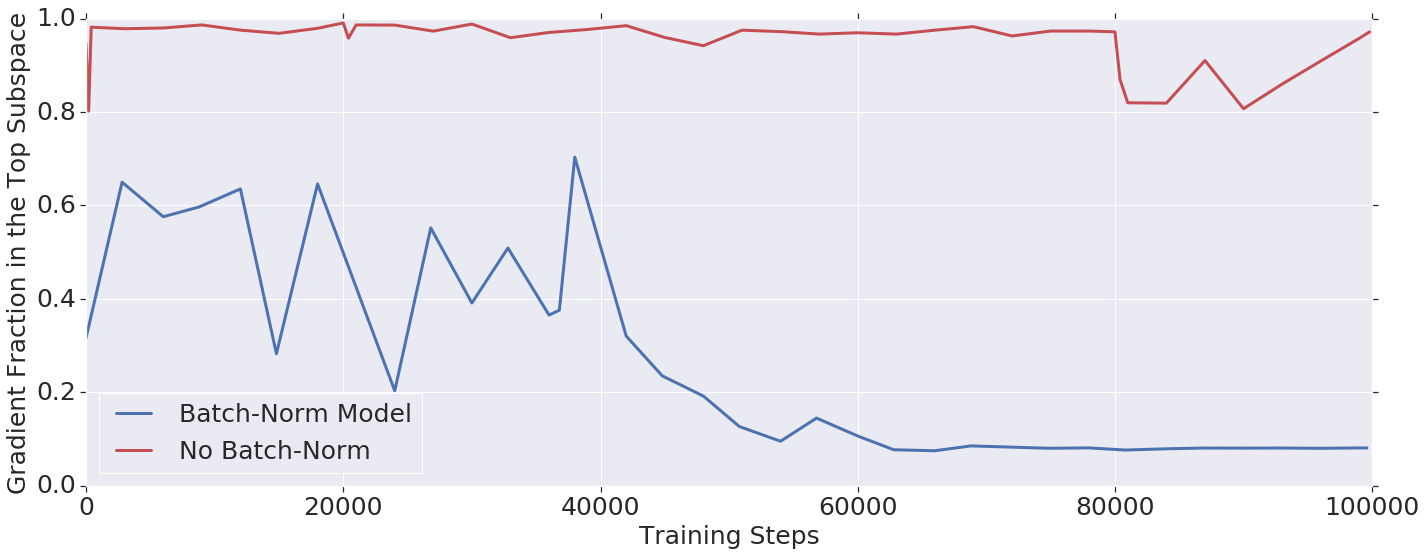}
\vspace{-0.5cm}
\caption{$\frac{\Vert P \nabla \loss(\hat{\theta}_i) \Vert_2^2}{\Vert \nabla \loss\hat{\theta}_i \Vert_2^2}$ for a Resnet-32. Here $P$ is the projection operator to the subspace spanned by the $10$ most dominant eigenvectors of $\nabla^2 \loss (\hat{\theta}_i)$. Almost all the variance of the gradient of the non-BN model is in this subspace. \label{fig:gradient_figure}}
\end{figure}

Since almost all of the gradient energy is in the very few outlier directions, the projection of the gradient in the complement of this subspace is minuscule. Thus, most gradient updates do not optimize the model in the flatter directions of the loss. As argued earlier, a significant portion of the loss comes from these flatter directions and a large fraction of the path towards the optimum lies in these subspaces. The fact that the gradient vanishes in these directions forces the training to be very slow. 

Stated differently, the argument above suggest that, in non-BN networks, the gradient is uninformative for optimization, i.e., moving towards the (negative) gradient hardly takes us closer to the optimum $\theta^*$. To support this argument, we plot the normalized inner product between the path towards the optimum, $\theta^* - \hat{\theta}_t$, \footnote{We use the parameter at the end of the training as a surrogate for $\theta^*$.} and the gradients, $\nabla \loss(\hat{\theta}_t)$, throughout the training trajectory (Figure \ref{fig:GradInf}). The figure suggests that the direction given by the gradient is almost orthogonal to the path towards the optimum. Moreover, the plot suggests that in BN networks, where the gradient is less concentrated in the high-curvature directions, the situation is significantly better.

In Appendix \ref{app:gradient_capture}, we study the relationship of the Hessian outliers with the concentration of the gradient phenomenon in a simple stochastic quadratic model. We show that when the model is optimized via stochastic gradients, outliers in the Hessian spectrum over-influence the gradient and cause it to concentrate in their direction. As argued above, gradient concentration is detrimental to the optimization process. Therefore, this result suggests yet another way in which outlier eigenvalues in $H$ disrupt training. 


\begin{figure}[h]
\includegraphics[width=\textwidth]{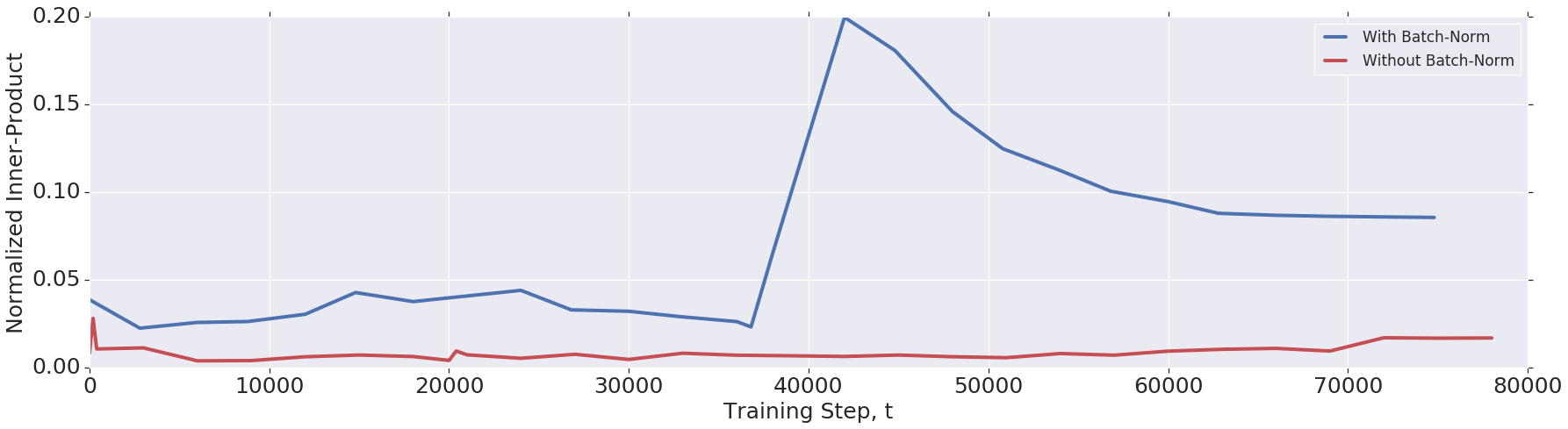}
\vspace{-0.5cm}
\caption{Normalized inner product between $\nabla \loss(\theta_t)$ and $\theta_t - \theta^*$ throughout the optimization for a Resnet-32 model. \label{fig:GradInf}}
\end{figure}

\subsection{Testing our hypothesis}
\label{sec:testing_hypothesis}
Our hypothesis that batch norm suppresses outliers, and hence speeds up training, is simple enough to allow us to make predictions based on it. The original batch normalization paper \cite{ioffe2015batch} observed that the normalization parameters of BN, $\sigma_B$ and $\mu_B$, have to be computed (and back-propagated through) using the mini-batch. If $\sigma_B, \mu_B$ are computed using the complete dataset, the training becomes slow and unstable. Therefore, we postulate that when $\sigma_B$ and $\mu_B$ are calculated from the population (i.e. full-batch) statistics, the outliers persist in the spectrum. 

To test our prediction, we train a Resnet-32 on Cifar-10 once using mini-batch normalization constants (denoted by mini-batch-BN network), and once using full-batch normalization constants (denoted by full-batch-BN network). The model trained with full-batch statistics trains much slower (Appendix \ref{app:population_bn}). Figure \ref{fig:population_statistics_train} compares the spectrum of the two networks in the early stages of the training (the behavior is the same during the rest of training). The plot suggests strong outliers are present in the spectrum with full-batch-BN. This observation supports our hypothesis. Moreover, we observe that the magnitude of the largest eigenvalue of the Hessian in between the two models is roughly the same throughout the training. Given that full-batch-BN network trains much more slowly, this observation shows that analyses based on the top eigenvalue of the Hessian do not provide the full-picture of the optimization hardness.

\begin{figure}[h]
\includegraphics[width=\textwidth]{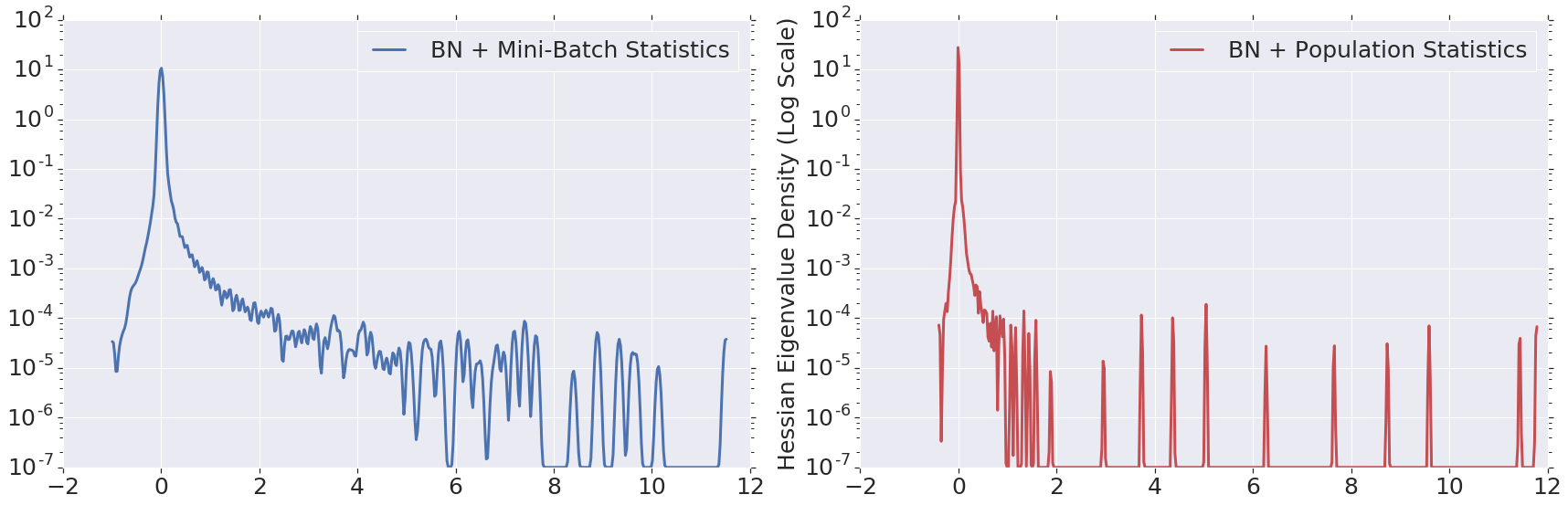}
\vspace{-0.5cm}
\caption{The Hessian spectrum for a Resnet-32 after $6k$ steps. The network on the left is trained with BN and mini-batch statistics. The network on the right is trained with population statistics. \label{fig:population_statistics_train}}
\end{figure}

\section{Conclusion}

We presented tools from advanced numerical analysis that allow for computing the spectrum of the Hessian of deep neural networks in an extremely accurate and scalable manner. We believe this tool is valuable for the research community as it gives a comprehensive view of the local geometry of the loss. This information can be used to further our understanding of neural networks.

We used this toolbox to study how the loss landscape locally evolves throughout the optimization. We uncovered surprising phenomena, some of which run contrary to the widely held beliefs in the machine learning community. In addition, we provided simple and clear answers to how batch-normalization speeds up training. We believe that BN is only one of the many architecture choices that can be studied using our framework. Studying these other architecture choices can be an interesting avenue for future research.

\bibliography{citations}
\bibliographystyle{plain}

\newpage
\clearpage
\cleardoublepage
\appendix
\newpage
\section{Concentration of Quadratic Forms}
\label{app:more_concentration}
The following lemma is one result on the concentration of quadratic forms:
\begin{lemma}[Concentration of Quadratic Forms, \cite{bellec2014concentration}] Let $\zeta \sim N(0, \sigma^2 I_n)$. Let $A\in \R^{n\times n}$ be any matrix. Then, $\forall x > 0$,
\label{lem:concentration}
\begin{align} 
P(\zeta^T A \zeta - \E{\zeta^T A \zeta} > 2\sigma^2 \Vert A \Vert_F \sqrt{x} + 2\sigma^2 \Vert A\Vert_2 x) \nonumber \leq \exp(-x).
\end{align}
\end{lemma}
We are now ready to prove Claim \ref{cor:phi_concentrate}.
\begin{proof}
Consider the block-diagonal matrix $A = \oplus_{i=1}^k  f(H; t, \sigma^2)$. Then, $\widehat{\phi}_{\sigma}(t) = w^T A w$ where $w$ is the concatenation of the $k$ realizations of $v$ divided by $\sqrt{k}$. Now observe that $w$ is i.i.d $\mathcal{N}(0, \frac{1}{kn})$. Therefore, by Lemma \ref{lem:concentration}, 
\begin{equation*}
P\bigg(|\phis(t) - \widehat{\phi}_{\sigma}(t)| > \frac{2\Vert A \Vert_F}{k n} \sqrt{x} +  \frac{2\Vert A \Vert_2}{k n} x \bigg) \leq 2\exp(-x).
\end{equation*}
Now observe that $\Vert A \Vert_F = \sqrt{k}\Vert f(H; t, \sigma^2) \Vert_F$ and $\Vert A \Vert_2 = \Vert f(H; t, \sigma^2)\Vert_2$. Therefore, we get 
\begin{equation}\label{eq:simple_format}
P\bigg(|\phis(t) - \widehat{\phi}_{\sigma}(t)| > \frac{2a}{n\sqrt{k}} \sqrt{x} +  \frac{2b}{k n} x \bigg) \leq 2\exp(-x).
\end{equation}

From \eqref{eq:simple_format} is clear that the bound deteriorates as $a$ and $b$ increase. Since $f(\cdot)$ is the Gaussian density, we know $b \leq \frac{1}{\sqrt{2 \pi} \sigma}$ and $a \leq \sqrt{n} b$. Substituting these worst case scenario values in \eqref{eq:simple_format}, we get
\begin{equation}
P\bigg(|\phis(t) - \widehat{\phi}_{\sigma}(t)| > \sqrt{\frac{2}{\pi \sigma^2}} (\sqrt{\frac{x}{nk}} + \frac{x}{nk}) \bigg) \leq 2\exp(-x).
\end{equation}
This proves our assertion.
\end{proof}

Figure \ref{fig:tail_bound} shows how $\epsilon(x)$ changes with respect to probability bound $2 \exp(-x)$ in the worst case bound $\eqref{eq:worst_case}$. We can see that even with modest values of $k$, we can achieve tight bounds on $\epsilon$ with high probability. 

\begin{figure} 
\includegraphics[width=\textwidth]{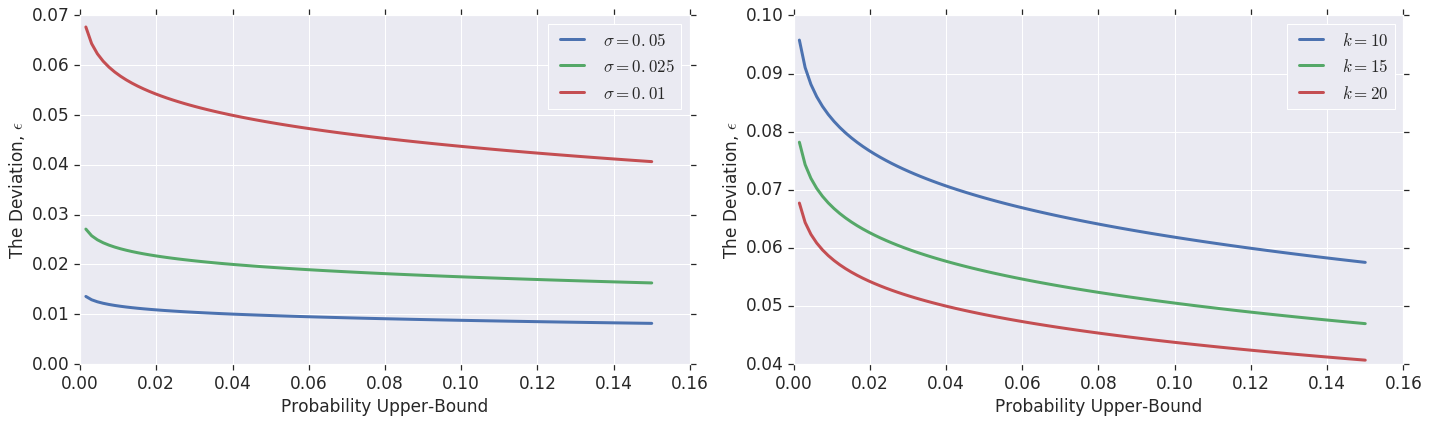}
\caption{Examination of the worst-case tail bound for a network with $n = 5\times 10^5$ parameters. Left figure: we set $k=20$ and change the kernel parameter $\sigma$. Right figure: we set $\sigma =0.01$ and change $k$.\label{fig:tail_bound}}
\end{figure}

\section{Numerical Verification on Small Models}
\label{app:verify}
Figure \ref{fig:nodes} shows how fast $\phisv$ converges to $\widehat{\phi}^{(v)}(t)$ as $m$ increases in terms of total variation ($L_1$) distance. 

\begin{figure}[h]
\includegraphics[width=\textwidth]{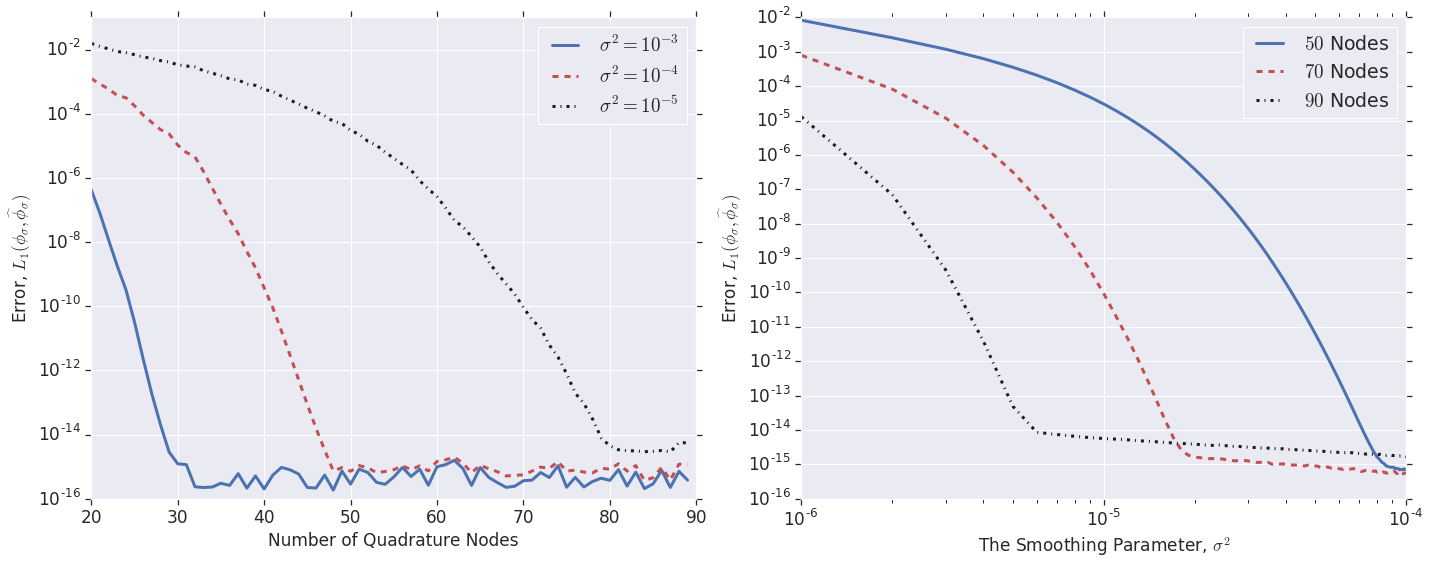}
\vspace{-0.5cm}
\caption{The left plot shows the accuracy of the Gaussian quadrature approximate as the number of nodes increases. A degree $80$ approximation achieves double-precision accuracy of $10^{-14}$. The right plot shows how the accuracy changes as the kernel width, $\sigma^2$, increases. For our large-scale experiments, we use $\sigma^2 = 10^{-5}$ and $90$ quadrature nodes. \label{fig:nodes}}
\end{figure}

Before going to large scale experiments, we empirically demonstrate the accuracy of our proposed framework on a small model where the Hessian eigenvalues can be computed exactly. Let's consider a feed-forward neural network trained on $1000$ MNIST examples with $1$ hidden layer of size $20$, corresponding to $n=15910$ parameters. The Hessian of networks of this type were studied earlier in \cite{sagun2017empirical} where it was shown that, after training, the spectrum consists of a bulk near zero and a few outlier eigenvalues. In our example, the range $[-0.2, 0.4]$ roughly corresponds to the bulk and $(0.4, 10)$ corresponds to the outlier eigenvalues. Figures \ref{fig:verification} and \ref{fig:verification_outliers} compare our estimates with the exact smoothed density on each of these intervals. Our results show that with a modest number of quadrature points (90 here) we are able to approximate the density extremely well. 
Our proposed framework achieves $L_1(\phis, \widehat{\phi}_{\sigma}) \approx 0.0012$ which corresponds to an extremely accurate solution. As demonstrated in Figure \ref{fig:verification_outliers}, our estimator detects the presence of outlier eigenvalues. Therefore, the information at the edges of $\phis$ is also recovered.

\begin{figure}[h]
\includegraphics[width=\textwidth]{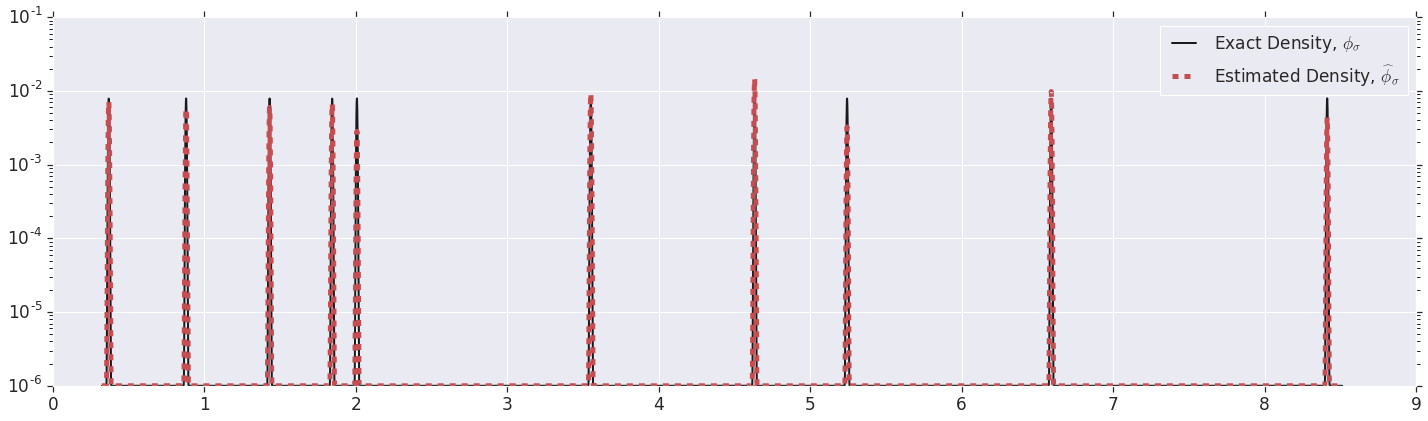}
\vspace{-0.5cm}
\caption{Comparison of the estimated smoothed density (dashed) and the exact smoothed density (solid) in the interval $[0.4, +\inf)$. We use $\sigma^2 = 10^{-5}, k=10$ and degree $90$ quadrature. \label{fig:verification_outliers}}
\end{figure}

\section{Implementation Details}
\label{app:implementation}

The implementation of Algorithm \ref{meta_alg} for a single machine is straightforward and can be done in a few lines of code. Scaling it to run on a 27 million parameter Inception V3 \cite{szegedy2016rethinking} on ImageNet (where we performed our largest scale experiments) requires a significant engineering effort.

The major component is a distributed Lanczos algorithm. Because modern deep learning models and datasets are so large, it is important to be able to run Hessian-vector products in parallel across multiple machines. At each iteration of the Lanczos algorithm, we need to compute a Hessian-vector product on the entire dataset. To do so, we split the data across all our workers (each one of which is endowed with one or more GPUs), each worker computes mini-batch Hessian-vector products, and these products are summed globally in an accumulator. Once worker $i$ is done on its partition of the data, it signals via semaphore $i$ to the chief that it is done. When all workers are done, the chief computes completes the Lanczos iteration by applying a QR orthogonalization step to total Hessian-vector product. When the chief is done, it writes the result to shared memory and raises all the semaphores to signal to the workers to start on a new iteration.

For the Hessian-vector products, we are careful to eliminate all non-determinism from the computation, including potential subsampling from the data, shuffle order (this affects e.g., batch normalization), random number seeds for dropout and data augmentation, parallel threads consuming data elements for summaries etc. Otherwise, it is  unclear what matrix the Lanczos iteration is actually using.

Although GPUs typically run in single precision, it is important to perform the Hessian-vector accumulation in double precision. Similarly, we run the orthogonalization in the Lanczos algorithm in double precision. TensorFlow variable updates are not atomic by default, so it is important to turn on locking, especially on the accumulators. TensorFlow lacks communication capability between workers, so the coordination via semaphores (untrainable tf.Variables) is crude but necessary.

For a CIFAR-10, on 10 Tesla P100 GPUs, it takes about an hour to compute 90 Lanczos iterations. For ImageNet, a Resnet-18 takes about 20 hours to run 90 Lanczos iterations. An Inception V3 takes far longer, at about 3 days, due to needing to use 2 GPUs per worker to fit the computation graph. We were unable to run any larger models due to an unexpected OOM bugs in TensorFlow. It should be straightforward to obtain a 50-100\% speedup -- we use the default TensorFlow parameter server setup, and one could easily reduce wasteful network transfers of model parameters  from parameter servers for every mini-batch, and conversely from transferring every mini-batch Hessian-vector product back to the parameter servers. We made no attempt to optimize these variable placement issues.

For the largest models, TensorFlow graph optimizations via Grappler can dramatically increase peak GPU memory usage, and we found it necessary to manage these carefully.

\section{Comparison with Other Spectrum Estimation Methods}
\label{app:comparison}
There is an extensive literature on estimation of spectrum of large matrices. A large fraction of the algorithms in this literature relay on explicit polynomial approximations to $f$. To be more specific, these methods approximate $f(\cdot, t, \sigma^2)$ with a polynomial of degree $m$, $g_{m}(\cdot)$. In step I of Algorithm \ref{meta_alg}, $\phisv(t)$ is approximated by
\begin{align}\label{eqn:polynomial}
\phip^{(v)}(t) := \sum_{i = 1}^n \beta_i^2 g_m(\lambda_i).
\end{align}
If $g_m(\cdot)$ is a good approximation for $f(\cdot; t, \sigma^2)$, we expect $\phip^{(v)}(t) \approx \phisv(t)$.

Since $g_m$ is a polynomial, \eqref{eqn:polynomial} can be exactly evaluated as soon as 
\begin{align}
\mu_j^{(v)}\equiv \sum_{i=1}^n \beta_i^2 \lambda_i^j, \qquad 1 \leq j \leq m
\end{align}
are known. Note that by definition, 
\begin{align*}
\mu_j^{(v)} = \sum_{i=1}^n (v^T q_i)^2 \lambda_i^j = v^T Q\Lambda^{j}Q^Tv = v^T H^j v 
\end{align*}
Therefore, if done carefully, $\{\mu_j^{(v)}\}_{j=1}^m$ can be computed by performing $m$ Hessian-vector products in total. Hence, by performing $km$ Hessian-vector products one can run Algorithm \ref{meta_alg} with $k$ different realizations of $v$.

This approximation framework is arguably simpler than Gaussian quadrature method as it does not have to cope with complexities of Lanczos algorithm. Therefore, it is has been extensively used in the numerical linear algebra literature. The polynomial approximation step is usually done via Chebyshev polynomials. This class of polynomials enjoy strong computational and theoretical properties that make them suitable for approximating smooth functions. For more details on Chebyshev polynomials we refer the reader to \cite{gil2007numerical}.

Recently, there has been a proposal to use Chebyshev approximation for estimating the Hessian spectrum for deep networks \cite{adams2018estimating}. For completeness, we compare the performance of this algorithm with the Gaussian quadrature rule on the feed-forward network defined earlier.

Figure \ref{fig:Chebyshev} shows the performance of the Chebyshev method in approximating $\phis(t)$. The hyper-parameters are selected such that the performance of the Chebyshev method in Figure \ref{fig:Chebyshev} is directly comparable with the performance of Gaussian quadrature in Figure \ref{fig:verification}. In particular, both approximations take the same amount of computation (as measured by the number of Hessian-vector products) and they both use the same kernel width ($\sigma^2 = 10^{-5}$). As the figure shows, the Chebyshev method utterly fails to provide a decent approximation to the spectrum. As it can be seen from the figure, almost all of the details of the spectrum are masked by the artifacts of the polynomial approximation. In general, we expect the Chebyshev method to require orders of magnitude more Hessian-vector products to match the accuracy of the Gaussian quadrature.

It is not a surprise that explicit polynomial approximation fails to provide a good solution. For small kernel widths, extremely high order polynomials are necessary to approximate the kernel well. Figure \ref{fig:ChebyshevKernel} shows how well Chebyshev polynomials approximate the kernel $f$ with $\sigma^2 = 10^{-5}$. The figure suggests that even with a $500$ degree approximation, there is a significant difference between the polynomial approximation and the exact kernel.  

\begin{figure}[h]
\includegraphics[width=\textwidth]{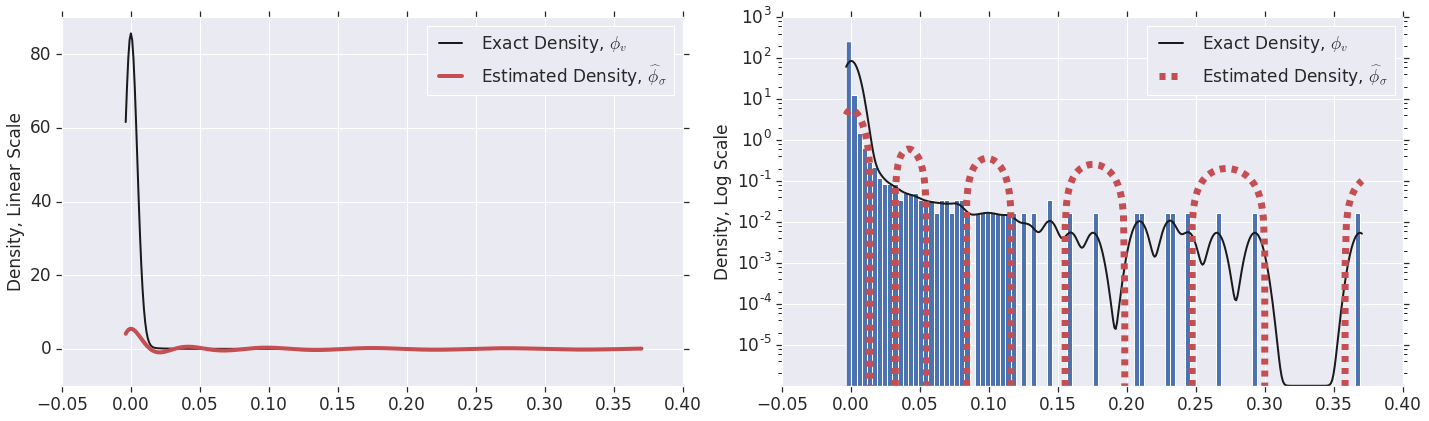}
\vspace{-0.5cm}
\caption{Estimated Hessian spectral density using Chebyshev approximation method for the feed-forward model. The left plot shows the densities in the linear scale and the right plot shows the densities in the log scale. Degree $90$ polynomial was used to estimate the density. $\sigma^2 = 10^{-5}$ was used as the kernel parameter. To factor out the effects of noise in moment estimation, exact eigenvalue moments were provided to the algorithm. \label{fig:Chebyshev}}
\end{figure}

\begin{figure}[h]
\includegraphics[width=\textwidth]{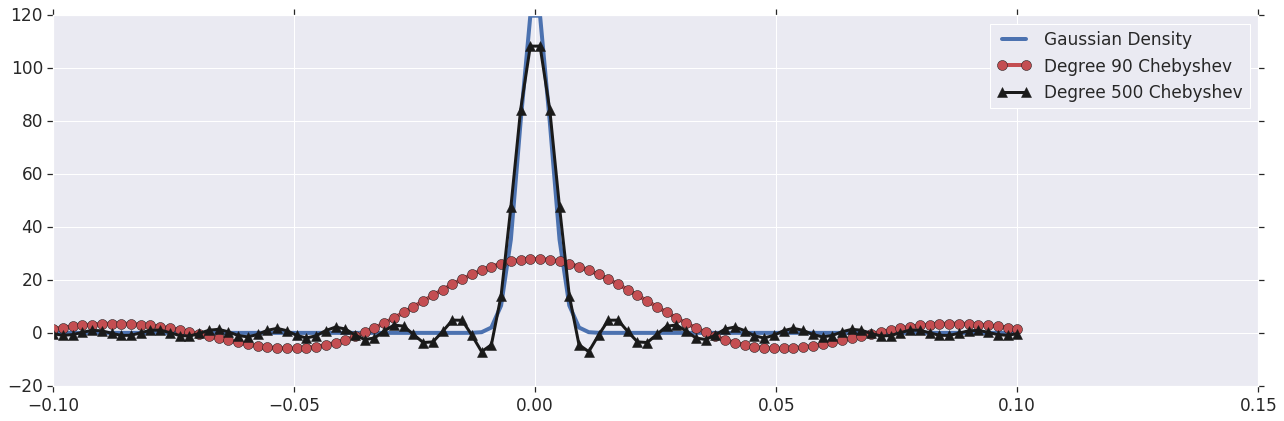}
\vspace{-0.5cm}
\caption{Demonstrating the quality of Chebyshev polynomial approximation to the Gaussian kernel with $\sigma^2 = 10^{-5}$. The plot suggests that approximations of order $500$ or more are necessary to achieve accurate results. Such high order approximations are statistically unstable and extremely computationally expensive. \label{fig:ChebyshevKernel}}
\end{figure}

\section{Gradient Concentration in the Quadratic Case}
\label{app:gradient_capture}

In this section, we theoretically show the phenomenon of gradient concentration on a simple quadratic loss function with stochastic gradient descent. The loss function is of the form 
$$\loss(\theta) = \frac{1}{2} (\theta-\theta^*)^T H (\theta-\theta^*),
$$ 
where the ordered (in decreasing order) eigenpairs of $H$ are $(\lambda_i, q_i), i = 1, \cdots, n$ (implies $H q_i = \lambda_i q_i$) and the iteration starts at $\theta_0 \sim \mathcal{N}(0, I_n)$. We model the stochastic loss (from which we compute the gradients for SGD) as 
$$\hat{\loss}(\theta) = \frac{1}{2} (\theta - \theta^* + z)^T H (\theta - \theta^* + z),
$$
where $z$ is a random variable such that $\E{z} = 0$ and $\E{zz^T} = \Noise$. In order to understand gradient concentration, we look at the alignment of individual SGD updates with individual eigenvectors of the Hessian. We are now ready to prove the following theorem.

\begin{theorem}\label{thm:quadratic}
 Consider a single gradient descent iteration, $\theta_{t+1} = \theta_t - \eta \nabla \hat{\loss}$ with a constant learning rate $\eta \approx c/\lambda_1$ for a constant $c < 1$. Then,
 \begin{align}
     \E{\left< q_i, (\theta_{t+1} - \theta_t) \right>^2} \to \alpha \cdot \left( \frac{\lambda_i}{\lambda_1} \right)^2 \cdot (q_i^T \Noise q_i)
 \end{align}
for some sufficiently large constant $\alpha$ as $t \to \infty$.
\end{theorem}

\begin{proof}
Each stochastic gradient step has the form
$\theta_t = \theta_{t-1} - \eta H (\theta_{t-1} + z_{t-1})$.
Expanding the recurrence induced by gradient step over $t$ steps, we can write
$$\theta_t = (I_n - \eta H)^t \theta_0 - \eta H \sum_{j=0}^{t-1} (I_n - \eta H)^j z_{t-j-1}.
$$

Therefore a single update $\theta_{t+1} - \theta_t = -\eta H (\theta_t + z_t)$ can be expanded as

\begin{align*}
\theta_{t+1} - \theta_t = - \eta H [(I_n - \eta H)^t \theta_0 
- \eta H \sum_{j=0}^{t-1} (I_n - \eta H)^j z_{t-j-1} + z_t]
\end{align*}

We can write the above equation as $ -\eta H ( T_1 + T_2)$, where
\begin{align*}
\begin{split}
T_1 &= (I_n - \eta H)^t \theta_0 \\
T_2 &= - \eta H \sum_{j=0}^{t-1} (I_n - \eta H)^j z_{t-j-1} + z_t
\end{split}
\end{align*}

Consider the dot product of this update with one of the eigenvectors $q_i$. Clearly from the form of the update $\E{\left<q_i, \theta_{t+1} - \theta_t \right>} \xrightarrow{t \rightarrow \infty} 0$. We now quantify the variance of the update in the direction of $q_i$. Using the identity $H q_i = \lambda_i q_i$, it is easy to see that

\begin{align*}
\begin{split}
q_i^T \eta H T_1 &= \eta \lambda_i (1 - \eta \lambda_i)^t q_i^T \theta_0 \\
q_i^T \eta H T_2 &= - \eta^2 \lambda_i^2 \sum_{j=0}^{t-1} (1 - \eta \lambda_i)^j q_i^T z_{t-j-1} + \eta \lambda_i q_i^T z_t
\end{split}
\end{align*}

Squaring the sum of the two terms above and taking expectations, only the squared terms survive. We write the 
\begin{align*}
\begin{split}
 \E{\left< q_i, (\theta_{t+1} - \theta_t) \right>^2} =  \eta^2 \lambda_i^2 (1 - \eta \lambda_i)^{2t} 
 +  \left( \eta^4 \lambda_i^4 \sum_{j=0}^{t-1} (1 - \eta \lambda_i)^{2j} + \eta^2 \lambda_i^2 \right) \cdot (q_i^T \Noise q_i)
\end{split}
\end{align*}

As $t \to \infty$, the first term above goes to 0. This suggests that in the absence of noise in the gradients there is no reason to expect any alignment of the gradient updates with the eigenvectors of the Hessian. However, the second term (after some algebraic simplification) can be written as 
$$\frac{2 \eta^2 \lambda_i^2}{2 - \eta \lambda_i} \cdot (q_i^T \Noise q_i).
$$
Parameterizing $\eta = c/ \lambda_1$ completes the proof.


\end{proof}
A couple of observations are appropriate here. We can see that as the separation of eigenvalues increases, gradient updates align quadratically with the top eigenspaces. By manipulating the alignment of $\Noise$ with the top eigenspaces of $H$, we can dramatically change the concentration of updates. For example, if $\Noise$ was similar to $H$, the alignment with the top eigenspaces can be enhanced. If $\Noise$ was similar to $H^{-1}$, the alignment with the top eigenspaces can be diminished. We have seen that, even in practice, if we could control the noise in the gradients, we can hamper or improve optimization in significant ways.

\section{Experimental Details}
\label{app:details}

On CIFAR-10, our models of interest are:
\begin{description}
\item[Resnet-32:] This model is a standard Resnet-32 with 460k parameters. We train with SGD and a batch size of 128, and decay the learning from 0.1 by factors of 10 at step 40k, 60k, 80k. This attains a validation of 92\% with data augmentation (and around 85\% without)
\item[VGG-11:] This model is a slightly modified VGG-11 architecture. Instead of the enormous final fully connected layers, we are able to reduce these to 256 neurons with only a little degradation in validation accuracy (81\% vs 83\% with a 2048 size fully connected layers). We train with a constant SGD learning rate of 0.1, and a batch size of 128. This model has over 10 million parameters.
\end{description}
To ensure that our models have a finite local minimum, we introduce a small label smoothing of 0.1. This does not affect the validation accuracy; the only visible effect is that the lowest attained cross entropy loss is the entropy 0.509.

On ImageNet, our primary model of interest is Resnet-18. We use the model in the official TensorFlow Models repository \cite{inceptionaugmentation}. However, we train the model on $299\times299$ resolution images, in an asynchronous fashion on 50 GPUs with an exponentially decaying learning rate starting at 0.045 and batch size 32. This attains 71.9\% validation accuracy. This model has over 11 million parameters.

\section{Batch normalization with population statistics}
\label{app:population_bn}
The population loss experiment is quite difficult to run on CIFAR-10 (we were unable to make Inception V3 train in this way without using a tiny learning rate of $10^{-6}$). In particular, it is important to divide the learning rate by a factor of 100, and also to spend at least 400 steps at the start of optimization with a learning rate of 0: this allows the batch normalization population statistics to stabilize with a better initialization than the default mean of 0.0 and variance of 1.0. 
\begin{figure}[h]
\includegraphics[width=\textwidth]{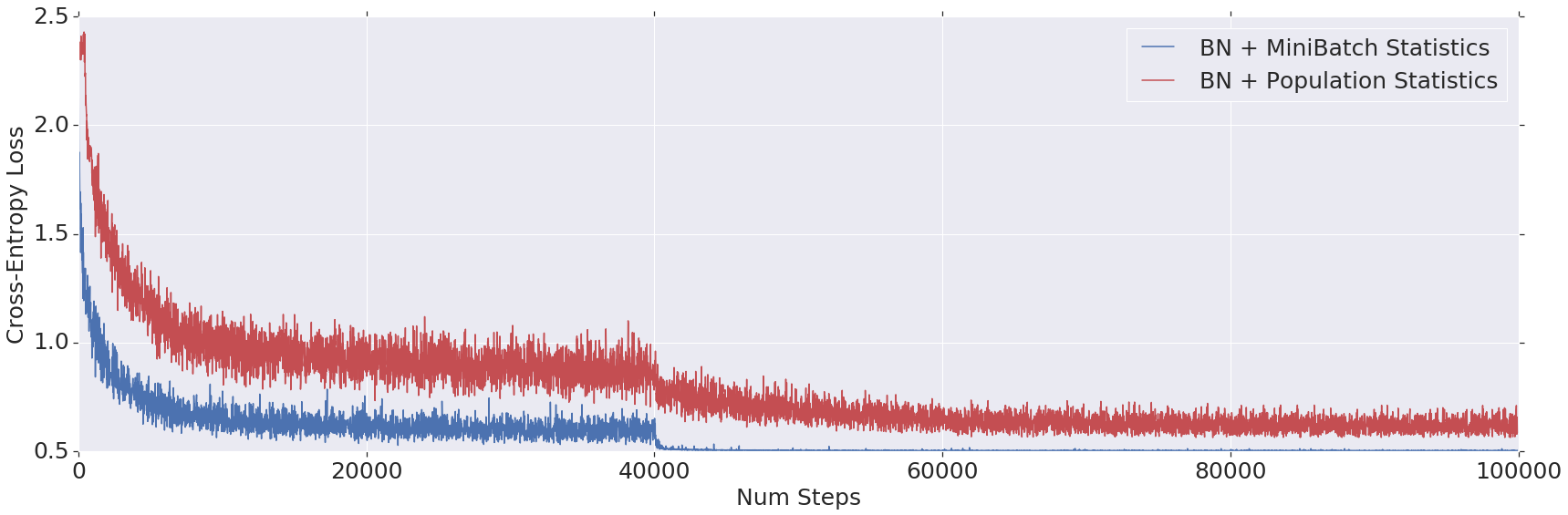}
\vspace{-0.5cm}
\caption{Optimization progress (in terms of loss) of batch normalization with mini-batch statistics and population statistics. \label{fig:bn_population_losses}}
\end{figure}

\end{document}